\documentclass[preprint,12pt]{elsarticle}
\usepackage{amsmath, amssymb, amsthm, bbm}
\usepackage{graphicx, wrapfig}
\usepackage{booktabs}
\usepackage{caption}
\usepackage{subcaption}
\usepackage{algorithm}
\usepackage{algpseudocode}
\usepackage{hyperref}
\usepackage{cleveref}
\usepackage{bm}
\usepackage{setspace}
\usepackage{multirow}
\usepackage{adjustbox}
\usepackage{array}     
\usepackage{tikz}
\usetikzlibrary{positioning, arrows.meta, shapes.geometric}
\usepackage{enumitem}
\usepackage{xcolor}
\definecolor{darkgreen}{rgb}{0.0, 0.5, 0.0}

\newcounter{protocol}%
\newcounter{algorithm saved}%

\makeatletter
\newenvironment{protocol}[1][htb]{%
    \renewcommand{\ALG@name}{Function}%
    \setcounter{algorithm saved}{\value{algorithm}} %
    \setcounter{algorithm}{\value{protocol}}%
    \begin{algorithm}[#1]%
    }{\end{algorithm}
    \setcounter{protocol}{\value{algorithm}}%
    \setcounter{algorithm}{\value{algorithm saved}}%
}
\makeatother

\newtheorem{theorem}{Theorem}
\newtheorem{definition}{Definition}
\newtheorem{lemma}{Lemma}
\newcommand{\VSPS}{\textsc{VSPS}}
\newcommand{\rvx}{X} 
\newcommand{\rvy}{Y} 
\newcommand{\rvz}{Z} 
\newcommand{\valx}{x} 
\newcommand{\valy}{y} 
\newcommand{\valz}{z} 

\newcommand{\Itrain}{\mathcal{I}_{\text{train}}}
\newcommand{\Ical}{\mathcal{I}_{\text{cal}}}
\newcommand{\Itest}{\mathcal{I}_{\text{test}}}
\newcommand{\Ival}{\mathcal{I}_{\text{val}}}

\journal{Pattern Recognition}

\begin{document}

\begin{frontmatter}



\title{Volume-Sorted Prediction Set: Efficient Conformal Prediction for Multi-Target Regression} 

\author[sye]{Rui Luo} 
\affiliation[sye]{organization={Department of Systems Engineering, City University of Hong Kong},
            city={Hong Kong SAR},
            country={China}}
\author[mn]{Zhixin Zhou}
\affiliation[mn]{organization={Alpha Benito Research},
            city={Los Angeles},
            country={America}}

\begin{abstract}
We introduce Volume-Sorted Prediction Set (\VSPS{}), a novel method for uncertainty quantification in multi-target regression that uses conditional normalizing flows with conformal calibration. This approach constructs flexible, non-convex predictive regions with guaranteed coverage probabilities, overcoming limitations of traditional methods. By learning a transformation where the conditional distribution of responses follows a known form, \VSPS{} identifies dense regions in the original space using the Jacobian determinant. This enables the creation of prediction regions that adapt to the true underlying distribution, focusing on areas of high probability density. Experimental results demonstrate that \VSPS{} produces smaller, more informative prediction regions while maintaining robust coverage guarantees, enhancing uncertainty modeling in complex, high-dimensional settings.
\end{abstract}



\begin{keyword}
Conformal prediction \sep uncertainty quantification \sep multi-target regression \sep adaptive prediction regions \sep conditional normalizing flows
\end{keyword}

\end{frontmatter}

\section{Introduction}
\label{sec:introduction}

In real-world applications, it is often required to estimate more than one response variable \cite{gonccalves2023regression, patel2024conformal, dheur2025multi}. Consider, for example, estimating the effects and side effects of a drug given the patient's demographic information and medical measurements. These two responses may be correlated in such a way that when the drug is effective, the side effects are more severe~\cite{schuell2005side, nakano2022deep}, and this relation might not be linear.

In high-stakes settings, providing point predictions for the drug's effects and side effects is insufficient; the decision-maker must know the plausible effects for an individual \cite{hua2025mmdg}. The plausible effects can be represented as a region in the multidimensional space that covers a pre-specified proportion (e.g., 90\%) of the drug's possible outcomes. In the one-dimensional case, the region reduces to an interval, determined by lower and upper bounds for the response variable. The problem of constructing such prediction intervals is extensively investigated in the literature~\cite{koenker1978regression, izbicki2019flexible, gupta2022nested}. Extending this approach naively to the multi-target regression by estimating a prediction interval for each response separately results in a rectangle-shaped region. However, the true distribution of the response variables can have an arbitrary, non-convex shape, making the predicted region overly conservative and not reflective of the true underlying uncertainty.

A better approach is to predict both outcome variables jointly, encouraging the model to exclude unlikely combinations of the two from the predicted region. Existing methods, such as Directional Quantile Regression (DQR)~\cite{bovcek2017directional} and Vector Quantile Regression (VQR)~\cite{carlier2016vector}, attempt to address this by modeling joint distributions. However, these methods often impose restrictive assumptions on the shape or structure of the predictive regions, leading to either overly conservative estimates or inflexibility in capturing complex dependencies. Spherically Transformed Directional Quantile Regression (ST-DQR) \cite{feldman2023calibrated} overcomes these limitations by computing the directional quantiles in a latent space trained by a Variational Autoencoder (VAE). Nevertheless, this approach treats each latent sample equally in the transformation, thereby forfeiting some crucial information.


In this work, we develop a novel approach to enhance the performance of multi-target conformal prediction by identifying dense regions of the conditional distribution $\rvy|\rvx$. To achieve this, we train a conditional normalizing flow, denoted as $f_\phi$, such that the transformed variable $\rvz = f_\phi(\rvy, \rvx)$ approximately follows a known distribution, typically a multivariate normal distribution.

In this framework, a unit volume in the transformed latent space $\mathcal{Z}$ around a point $\valz$ corresponds to $\left|\det\left(\frac{\partial f_\phi^{-1}(\valz, \valx)}{\partial \valz}\right)\right|$ units in the original response space $\mathcal{Y}$, where $\frac{\partial f_\phi^{-1}(\valz, \valx)}{\partial \valz}$ represents the Jacobian matrix of the inverse transformation $f_\phi^{-1}$ with respect to $\valz$. By sampling a set of points $\{\valz^m\}_{m=1}^M$ from the latent space and mapping them back to the original space via $\valy^m = f_\phi^{-1}(\valz^m, \valx)$, we observe that points with larger values of $\left|\det\left(\frac{\partial f_\phi(\valy, \valx)}{\partial \valy}\right)\right|$ correspond to regions with larger volumes in $\mathcal{Y}$. Consequently, the prediction set should encompass regions near such $\valy^m$ to improve efficiency and informativeness.

This rationale underpins our proposed prediction set, termed the \textbf{Volume-Sorted Prediction Set} (\VSPS{}). While further implementation details require careful consideration, \VSPS{} leverages this fundamental principle to outperform existing approaches in multi-target prediction tasks. By focusing on volume-sorted regions, \VSPS{} effectively captures the underlying distribution's complexity, resulting in more precise and efficient uncertainty quantification in multi-target regression settings.

\section{Related Work}
\label{sec:related_work}

\subsection{Conformal Prediction for Multi-Target Regression}

Conformal prediction \cite{vovk2005algorithmic, manokhin2022awesome} offers distribution-free coverage guarantees, with wide applications across single-target regression \cite{lei2014distribution, romano2019conformalized,izbicki2019flexible, gupta2022nested, luo2024conformal}, classification \cite{romano2020classification, luo2024game, liu2024c, luo2024entropy, huang2024conformal, luo2024trustworthy, huang2025conformal, zeng2025parametric, tawachi2025multidimensional}, and graph-related tasks \cite{luo2023anomalous, severo2023one, lunde2023conformal, luo2024conformalized, marandon2024conformal, luo2025conformal, wang2025enhancing}. However, for conformal prediction in multi-target regression, independently constructing prediction intervals for each output ignores inter-variable correlations, resulting in overly conservative and less informative prediction regions \cite{feldman2023calibrated}. Accurately building informative multidimensional prediction regions remains challenging.

Messoudi et al. \cite{messoudi2021copula} proposed transforming the non-conformity scores of each target variable into uniformly distributed random variables, and then using copula functions to model the dependencies among these random variables. This approach can better capture the correlations between target variables. However, the regions generated by this method are hyper-rectangular and often excessively large, failing to effectively capture the true underlying distribution of the responses.

Feldman et al. \cite{feldman2023calibrated} proposed a method called STDQR to create smaller prediction regions with more arbitrary shapes. Their paper combines deep learning with DQR to generate non-convex prediction regions, while extending conformal prediction to multivariate responses, ensuring prediction regions achieve user-specified coverage probabilities. However, STDQR employs a uniform quantile threshold for all test inputs and still produces relatively large prediction intervals.

Diquigiovanni et al. \cite{diquigiovanni2022conformal} proposed a set of conformal predictors for multi-target regression, generating finite-sample, valid or exact multivariate simultaneous prediction bands for multivariate function responses under mild assumptions of exchangeable regression pairs. However, this method is limited to i.i.d. functional data and cannot handle dependent cases like functional time series. Xu et al. \cite{xu2024conformal}  introduced MultiDimSPCI, a sequential prediction method that constructs ellipsoidal prediction regions for multivariate time series. This approach aims to provide valid coverage guarantees without relying on exchangeability assumptions. Nevertheless, it requires estimating sample covariance matrices, demanding large datasets and losing effectiveness in high dimensions due to decreased accuracy in matrix estimation.

Johnstone et al. \cite{johnstone2021conformal} explored the extension of conformal prediction to the case of multivariate responses by utilizing Mahalanobis distance as a consistency score for multivariate responses. 
Despite these advancements, constructing accurate and informative prediction regions in multidimensional space remains a significant challenge, with each method having its own limitations and trade-offs.
\subsection{Conditional Density Estimation}
Density estimation methods are used to improve conditional coverage in conformal regression \cite{colombo2024normalizing}. Rigollet and Vert \cite{rigollet2009optimal} introduced kernel density estimation (KDE) for constructing prediction sets, but these methods often led to unstable coverage probabilities, especially in higher dimensions \cite{lei2014distribution}. 
Existing methods struggle with high-dimensional and multimodal data. To address these issues, Samworth et al. \cite{samworth2010asymptotics} proposed Gaussian mixture models (GMMs) for conditional density estimation in multimodal settings. Izbicki et al. \cite{izbicki2019flexible, izbicki2022cd} combined split conformal prediction with highest density sets to improve conditional coverage, but the partitioning problem persisted. Sampson and Chan \cite{sampson2024flexible} addressed these challenges by introducing the KDE for highest predictive density (KDE-HPD) method, which avoids data partitioning and integrates with conditional density estimators.

Methods such as those proposed by Lei and Wasserman \cite{lei2014distribution}, Cai et al. \cite{cai2014adaptive}, and Wang et al. \cite{wang2023probabilistic} leverage conditional density estimation to enhance prediction sets, especially in heteroscedastic settings. Han et al. \cite{han2022split} applied KDE to construct asymmetric prediction bands but encountered difficulties with multimodal distributions. Sesia and Romano \cite{sesia2021conformal} proposed a histogram-based approach, but it is computationally expensive for high-dimensional data. CP2 \cite{plassier2025conditionally} introduced implicit conditional generative models to address issues like multimodality and heteroscedasticity, though ensuring coverage consistency in varying data regions remains a challenge.

\section{Preliminaries and Problem Setup}
\label{sec:preliminaries}

\subsection{Notations}

Let $\{(\valx_i, \valy_i)\}_{i \in \Itrain}$ be a training dataset, where $\valx_i \in \mathcal{X} \subseteq \mathbb{R}^p$ and $\valy_i \in \mathcal{Y} \subseteq \mathbb{R}^d$ are the input feature vector and output response vector defined on the input space $\mathcal{X}$ and output space $\mathcal{Y}$, respectively. We assume that all input-output pairs are independently drawn from an underlying distribution $P$, i.e., $(\rvx, \rvy) \sim P$.

Let $\{(\valx_i, \valy_i)\}_{i \in \Ical}$ be a calibration dataset, and $\{(\valx_i, \valy_i)\}_{i \in \Itest}$ be a test dataset. Suppose $\hat{C}(\valx) \subseteq \mathcal{Y}$ is a prediction region in $\mathcal{Y}$ given an input $\valx$.

Our goal is to construct trustworthy prediction regions for multi-target regression tasks, such that they satisfy a conformal coverage guarantee. Specifically, we say the prediction region $\hat{C}(\rvx)$ guarantees $(1 - \alpha)$ coverage if the following inequality holds:
\begin{equation}
P\left(\rvy_i \in \hat{C}(\rvx_i)\right) \geq 1 - \alpha, \quad \forall i \in \Itest.
\label{eq:coverage}
\end{equation}

\subsection{Conditional Normalizing Flows}

Conditional Normalizing Flows (CNFs) extend the concept of normalizing flows to model conditional distributions. Given an input $\valx \in \mathcal{X}$ and a target $\valy \in \mathcal{Y}$, CNFs learn the conditional distribution $p_{\rvy|\rvx}(\valy|\valx)$ using a conditional prior $p_{\rvz|\rvx}(\valz|\valx)$ and a bijective and smooth mapping $f_\phi: \mathcal{Y} \times \mathcal{X} \to \mathcal{Z}$, which is invertible with respect to $\mathcal{Y}$ and $\mathcal{Z}$. The conditional density is given by:
\begin{align}\label{eq:conditional_density}
    p_{\rvy|\rvx}(\valy|\valx) = p_{\rvz|\rvx}(\valz|\valx) \left| \det \left( \frac{\partial \valz}{\partial \valy} \right) \right| 
    \approx p_{\rvz|\rvx}(f_\phi(\valy, \valx) | \valx) \left| \det \left( \frac{\partial f_\phi(\valy, \valx)}{\partial \valy} \right) \right|.
\end{align}
The generative process involves sampling $\valz \sim p_{\rvz|\rvx}(\valz|\valx)$ from a simple base density conditioned on $\valx$, and then passing it through either a single bijective mapping or a sequence of bijective mappings, denoted as $f_\phi^{-1}(\valz, \valx)$. This allows for modeling multimodal conditional distributions in $\valy$.

The Jacobian determinant plays a crucial role in CNFs and their application to multi-target conformal prediction. Let $\valz = f_\phi(\valy, \valx)$ be the transformed variable that approximately follows a known distribution, typically a multivariate normal distribution. The absolute value of the determinant of the Jacobian, $\left| \det \left( \frac{\partial f_\phi(\valy, \valx)}{\partial \valy} \right) \right|$, represents the change in volume when mapping from the original response space $\mathcal{Y}$ to the latent space $\mathcal{Z}$.

A unit volume in $\mathcal{Y}$ around a point $\valy$ corresponds to $\left| \det \left( \frac{\partial f_\phi(\valy, \valx)}{\partial \valy} \right) \right|$ units in $\mathcal{Z}$. By sampling points $\valz_m$ from the latent space and mapping them back to the original space via $\valy_m = f_\phi^{-1}(\valz_m, \valx)$, we observe that points with smaller values of $\left| \det \left( \frac{\partial f_\phi(\valy, \valx)}{\partial \valy} \right)|_{\valy = \valy^m} \right|$ correspond to regions with larger volumes in $\mathcal{Y}$. In Section \ref{subsec:volume_computation}, we demonstrate how this property is leveraged in the proposed Volume-Sorted Prediction Set (\VSPS{}) method for multi-target conformal prediction.

\subsection{Conformal Prediction}

Conformal Prediction (CP) is a general framework to provide rigorous coverage guarantees. 
For multi-target regression, we aim to construct a set $\hat{C}(\rvx_i)$ that contains the true response $\rvy_i$ with probability at least $1-\alpha$ for all $i \in \Itest$, as indicated in (\ref{eq:coverage}). 

CP typically relies on a non-conformity scoring function $s(\valx, \valy)$ that quantifies the dissimilarity between an observation $(\valx, \valy)$ and the training data. Given a test point $\valx_{\text{test}}$, we compute non-conformity scores for the augmented dataset $\{(\valx_i, \valy_i)\}_{i \in \Ical} \cup \{(\valx_{\text{test}}, \valy)\}$:
\begin{equation}\label{eq:non-conformity}
S_i = s(\valx_i, \valy_i), \quad i \in \Ical, \quad S_{\valy, \text{test}} = s(\valx_{\text{test}}, \valy).
\end{equation}

The conformal prediction set is then defined as:
\begin{equation}\label{eq:set}
\hat{C}(\valx_{\text{test}}) = \left\{\valy : S_{\valy, \text{test}} \leq Q_{1-\alpha}(\{S_i\}_{i \in \Ical})\right\},
\end{equation}
where $Q_{1-\alpha}$ is the $(1-\alpha)(1+1/|\Ical|)$-th empirical quantile of $\{S_i\}_{i \in \Ical}$.

The primary challenge in multi-target conformal prediction is constructing efficient prediction sets that balance coverage and size. As illustrated in Section \ref{sec:methodology}, our \VSPS{} method leverages a conditional normalizing flow $f_\phi$ to transform $\rvy|\rvx$ into a known distribution. By identifying regions with larger volumes in the original space $\mathcal{Y}$ corresponding to larger values of $\left| \det \left( \frac{\partial f_\phi(\valy, \valx)}{\partial \valy} \right)\right|$, \VSPS{} constructs prediction sets that prioritize high-density regions. This approach enables efficient computation of $\hat{C}(\valx_{\text{test}})$, even for non-convex regions in high dimensions, resulting in more informative uncertainty quantification in multi-target regression.

\section{Methodology}
\label{sec:methodology}

\begin{figure}[t]
    \centering
    \includegraphics[width=0.9\textwidth]{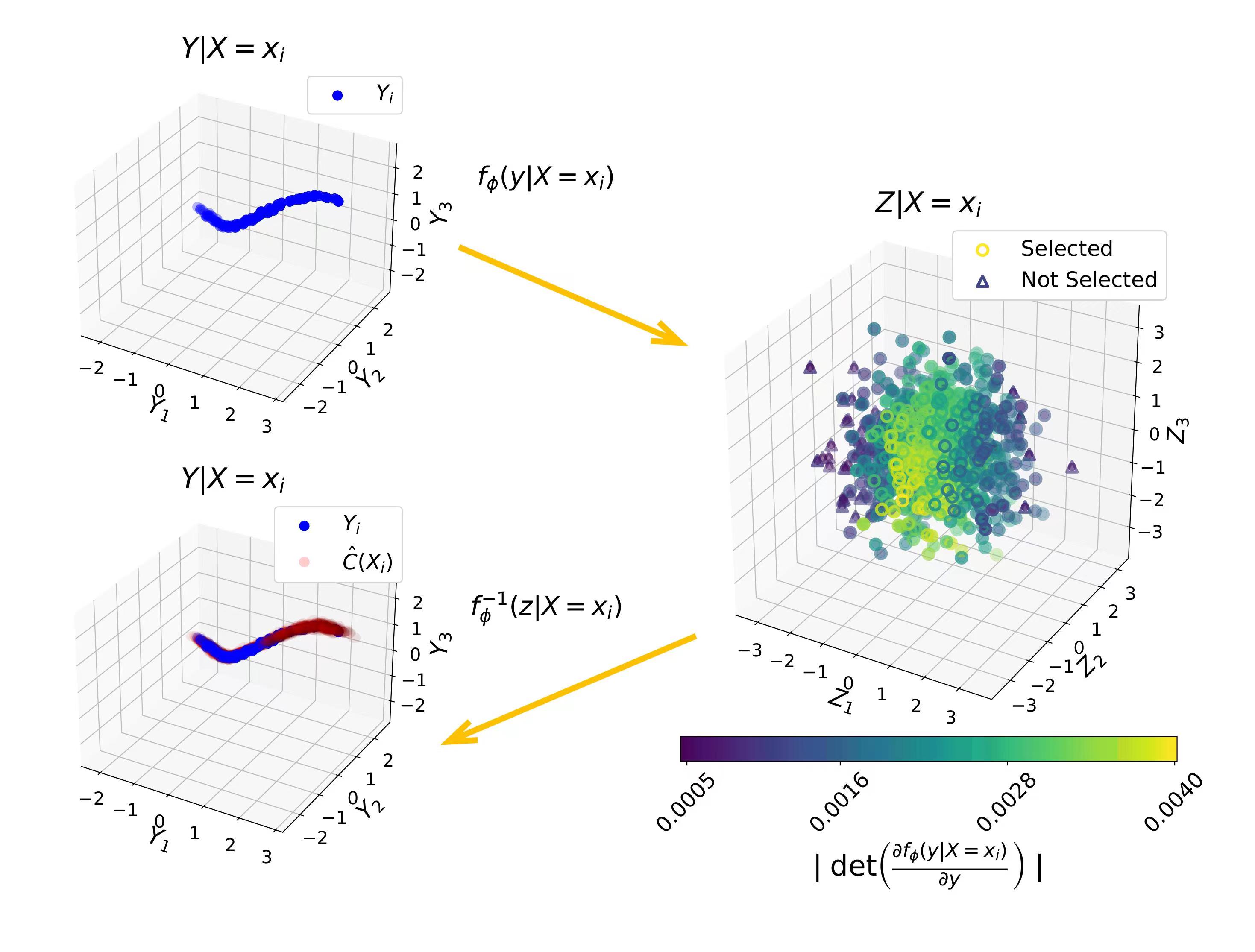}
    \caption{Scheme of the Volume-Sorted Prediction Set (VSPS) method. The CNF maps the original distribution to a standard normal distribution. Samples are generated, sorted, and selected based on their absolute Jacobian determinant. Top $K$ samples are transformed back to form the prediction region as balls of radius \(\gamma\) centered at the samples. \(\gamma\) is calibrated to ensure the desired $1-\alpha$ coverage, whereas the optimal $K$ is determined using a validation set. See Section \ref{sec:methodology} for details.}
    \label{fig:vsps_scheme}
\end{figure}

We introduce the Volume-Sorted Prediction Set (\VSPS{}) method for multi-target conformal prediction using conditional normalizing flows (CNFs), as illustrated in \Cref{fig:vsps_scheme}. \VSPS{} transforms the conditional distribution \(\rvy|\rvx\) into a simple latent distribution via a CNF \(f_\phi\) and constructs efficient prediction sets by prioritizing regions with larger \(\left| \det \left( \frac{\partial f_\phi(\valy, \valx)}{\partial \valy} \right) \right|\), corresponding to high-density areas in \(\mathcal{Y}\). The method integrates sampling via CNFs (\Cref{subsec:volume_computation}) with prediction region construction and conformal calibration (\Cref{subsec:adaptive_prediction}), along with optimal sample number selection, to guarantee the desired coverage while maximizing prediction set efficiency.

\subsection{Sampling Based on Conditional Normalizing Flow}
\label{subsec:volume_computation}

\begin{protocol}[!t]
\caption{Latent Variable Sampling and Mapping}
\begin{algorithmic}[1]
\Function{Sampling}{$\valx_{\text{test}}, M, f_\phi$}
    \State Sample latent variables:
    \[
    \{\valz^m\}_{m=1}^M \sim \mathcal{N}(0, I_d).
    \]
    \State Map latent samples back to the original space using the inverse transformation:
    \[
    \valy^m = f_\phi^{-1}(\valz^m, \valx_{\text{test}}), \quad \forall m = 1, \dots, M.
    \]
    \For{each sample \(\valy^m\)}
        \State Compute the Jacobian determinant:
        \[
        J^m = \left|\det \left(\frac{\partial f_\phi(\valy, \valx)}{\partial \valy} \right) \bigg|_{\valy = \valy^m, \valx=\valx_{\text{test}}} \right|.
        \]
    \EndFor
    \State Sort the samples \(\{\valy^m\}_{m=1}^M\) based on their Jacobian determinants \(J^m\) in descending order.
    \State Denote the sorted samples as \(\{\valy^{(m)}\}_{m=1}^M\), where \(m\) indicates the rank.
    \State \Return \((\valy^{(1)}, \valy^{(2)}, \dots, \valy^{(M)})\).
\EndFunction
\end{algorithmic}\label{alg:sampling}
\end{protocol}

We use CNFs to model complex target distributions in multi-target regression by learning a bijective and smooth mapping \(f_\phi: \mathcal{Y} \times \mathcal{X} \to \mathcal{Z}\), where \(\mathcal{Z}\) is a latent space following a simple distribution, e.g., \(\mathcal{N}(0, I_d)\).

Using $f_\phi$, the conditional density (\ref{eq:conditional_density}) can be approximated as:
\[
p_{\rvy|\rvx}(\valy|\valx) \approx p_{\rvz|\rvx}\big(f_\phi(\valy, \valx) \mid \valx\big) \cdot \left|\det \left( \frac{\partial f_\phi(\valy, \valx)}{\partial \valy} \right)\right|,
\]
where the term \(\left|\det \left( \frac{\partial f_\phi(\valy, \valx)}{\partial \valy} \right)\right|\) is the absolute value of the Jacobian determinant, quantifying the volume change induced by the transformation \(f_\phi\), whereas the original conditional distribution \(p_{\rvy|\rvx}(\valy|\valx)\) is transformed into a standard multivariate normal distribution \(\mathcal{N}(0, I_d)\) in the latent space \(\mathcal{Z}\). 

To leverage the high-density areas identified by the CNF, for a given test input \(\valx_{\text{test}}\), we propose the sampling procedure as described in Function \ref{alg:sampling}.

\paragraph{Simplification of Sample Weighing in the Transformed Space} Note that we do not need to weigh the samples using the conditional density as in (\ref{eq:conditional_density}). Instead, only the absolute value of the Jacobian determinants is used. This is because the base distribution is implicitly accounted for through the sampling process in the transformed space.

\subsection{Prediction Region Construction and Conformal Calibration}
\label{subsec:adaptive_prediction}

\begin{protocol}[!t]
\caption{Conformal Calibration}
\begin{algorithmic}[1]
\Function{Calibration}{$\{y^{(m)}\}_{m=1}^{K^*}, \mathcal{D}_{\text{cal}}, \alpha$}
\For{each $(\valx_i, \valy_i)$ in $\mathcal{D}_{\text{cal}}$}
    \State Define the distance to the nearest ball center as
    \[
    d_i = \min_{m=1,\dots, K^*} \| \valy_i - \valy^{(m)} \|_2.
    \]
\EndFor
\State Define the radius $\gamma$ as
\[
\gamma = \inf\left\{ t \in \mathbb{R} : \frac{1}{|\mathcal{I}_{\text{calib}}|+1} \sum_{i \in \mathcal{I}_{\text{calib}}} \mathbf{1}\{ d_i \le t \} \ge 1 - \alpha \right\}.
\]
\State \Return $\gamma$.
\EndFunction
\end{algorithmic}\label{alg:calibrate}
\end{protocol}

Given the selected samples \(\valy^{(1)}, \dots, \valy^{(K)}\) representing high-density areas, the prediction region is constructed as a union of balls centered at these samples. Specifically, the prediction region is defined as the union of balls with radius \(\gamma\):
\[
\hat{C}(\valx_{\text{test}}) = \bigcup_{m=1}^{K} B\left(\valy^{(m)}, \gamma\right),
\]
where \(B\left(\valy^{(m)}, \gamma\right)\) represents a ball centered at the sample \(\valy^{(m)}\) with radius \(\gamma\). Formally, the ball is given by:
\[
B\left(\valy^{(m)}, \gamma\right) = \left\{\valy \in \mathcal{Y} : \|\valy - \valy^{(m)}\|_2 \leq \gamma \right\},
\]
where \(\|\cdot\|_2\) denotes the Euclidean norm. This construction ensures that the prediction region is both continuous and interpretable within the response space.
The construction of ball-shaped prediction regions, with a focus on high-density areas, aligns with the approaches in \cite{wang2023probabilistic, feldman2023calibrated, zheng2024optimizing}, though more complex shapes could also be explored \cite{tumu2024multi, thurin2025optimal}.

Next, we describe the procedure for calibrating the radius \(\gamma\) to guarantee the desired coverage level (\ref{eq:coverage}):
\[
\mathbb{P}(\rvy \in \hat{C}(\rvx)) \geq 1 - \alpha.
\]

Intuitively, if \(\gamma\) is too small and leads to undercoverage, the balls will be expanded; if \(\gamma\) is too large, the balls will be shrunk. Using split conformal prediction, \(\gamma\) is chosen as the quantile of the minimum distances between each calibration data point and its nearest ball center, as outlined in Function \ref{alg:calibrate}.

To further enhance the efficiency of the prediction region, we propose selecting the optimal number of samples \(K\) using a separate validation set, drawing inspiration from the Validity First Conformal Prediction (VFCP) framework in \cite{yang2024selection, luo2024weighted}, as shown in Algorithm \ref{alg:selection}.
The complete VSPS method is illustrated in Algorithm~\ref{alg:vsps}.

\begin{algorithm}[t!]
\small
\caption{Optimal $K^*$ Selection}
\label{alg:selection}
\begin{algorithmic}[1]
\Require Validation set $\mathcal{D}_{\text{val}} = \{(\valx_i, \valy_i)\}_{i \in \Ival}$, calibration set \(\mathcal{D}_{\text{cal}} = \{(\valx_i, \valy_i)\}_{i \in \Ical}\), trained CNF model $f_\phi$, coverage level $1-\alpha$, sample size $M$.
\Ensure The optimal number of selected top samples $K^*$.

\State Initialize \(K^*\leftarrow M\) and \(\text{Size}_{\min}\leftarrow \infty\)

\rlap{\hbox{\textcolor{darkgreen}{\Comment{Precompute the \(M\) samples for each validation point (vectorized over \(\mathcal{D}_{\text{val}}\))}}}}
\For{each \((\valx_i,\valy_i) \in \mathcal{D}_{\text{val}}\)}
    \State Compute \(\valy_i^{(1)}, \dots, \valy_i^{(M)} \leftarrow\) \Call{Sampling}{$\valx_i, M, f_\phi$}
\EndFor

\rlap{\hbox{\textcolor{darkgreen}{\Comment{Calibrate and record volumes of prediction regions.}}}}
\For{\(K = 1,\dots,M\)}
    
    \State For every \(i\in \Ival\), compute calibrated radius 
    $\gamma^{(K)}_i =$ \Call{Calibration}{$\{\valy_i^{(m)}\}_{m=1}^{K}, \mathcal{D}_{\text{cal}}, \alpha$}
    \State Construct prediction regions for all \(i\in \Ival\):
    \[
    \hat{C}^{(K)}(\valx_i) \;=\; \bigcup_{m=1}^{K} B\Bigl(\valy_i^{(m)}, \gamma^{(K)}_i\Bigr)
    \]
    \State Compute the average volume of prediction regions (\(\text{Size}\)):
    \[
    \text{Size}_K \;=\; \frac{1}{|\mathcal{D}_{\text{val}}|}\sum_{i\in \Ival}\text{Volume}\Bigl(\hat{C}^{(K)}(\valx_i)\Bigr)
    \]
    \rlap{\hbox{\textcolor{darkgreen}{\Comment{Update best \(K^*\) based on prediction region size}}}}
    \If{\(\text{Size}_K < \text{Size}_{\min}\)}
        \State Update \(K^* \leftarrow K\) and \(\text{Size}_{\min} \leftarrow \text{Size}_K\)
    \EndIf
\EndFor
\State \Return \(K^*\)

        
\end{algorithmic}
\end{algorithm}

\paragraph{Comparison with ST-DQR \cite{feldman2023calibrated}}

Spherically Transformed Directional Quantile Regression (ST-DQR) is a competitive method that addresses the convexity restriction of multivariate quantile regression models. However, compared to ST-DQR, which relies on Conditional Variational Autoencoders (CVAEs), the proposed \VSPS{} offers several new perspectives:

\begin{enumerate}[topsep=0pt, parsep=0pt, partopsep=0pt, itemsep=2pt]
    \item \textbf{Explicit Tractable Density Evaluation:} \VSPS{} enables exact density evaluations, avoiding the reliance on lower bounds as in CVAEs.
    
    \item \textbf{Informative Jacobian Determinants:} The absolute values of the Jacobian determinants provide meaningful information about the significance of different samples, enabling weighted sampling during the construction of prediction regions.
    
    \item \textbf{Full Invertibility:} \VSPS{} eliminates the need for training a separate decoder, ensuring lossless transformations between the original and latent spaces for precise modeling.
    
    \item \textbf{No Quantile Regression Needed:} Unlike ST-DQR, \VSPS{} does not require training a separate quantile regression model. Instead, the parameterized calibration of the volume-sorted prediction set naturally satisfies the requirements for constructing continuous prediction regions.
\end{enumerate}

\paragraph{Comparison with CONTRA~\cite{fang2025contra}}

The proposed \VSPS{} also differs from a concurrent work CONTRA in the following ways:

\begin{enumerate}[topsep=0pt, parsep=0pt, partopsep=0pt, itemsep=2pt]
    \item \textbf{Prediction Region Construction in Response Space:} \VSPS{} constructs prediction regions as unions of balls in the response space \(\mathcal{Y}\), whereas CONTRA operates in the latent space \(\mathcal{Z}\). Although the homeomorphism of the normalizing flow ensures that CONTRA constructs closed and connected prediction regions, in practice, Monte Carlo sampling is often required to approximate the boundary in the latent space. This introduces additional numerical approximations and computational overhead. More importantly, CONTRA does not utilize the Jacobian determinant information, making it less effective in constructing efficient prediction sets.
    
    \item \textbf{Additional Parameter for Sample Selection:} \VSPS{} introduces a new parameter \(K\) for selecting samples in high-density areas. This additional parameter allows for more efficient and flexible construction of prediction regions.
\end{enumerate}

\begin{algorithm}[h]
\small
\caption{Volume-Sorted Prediction Set (\VSPS{})}
\label{alg:vsps}
\begin{algorithmic}[1]
\Require Test input $\valx_{\text{test}}$, calibration set \(\mathcal{D}_{\text{cal}} = \{(\valx_i, \valy_i)\}_{i \in \Ical}\), trained CNF model $f_\phi$, coverage level $1-\alpha$, sample size $M$, the optimal number of selected top samples $K^*$ (obtained from Algorithm \ref{alg:selection}).
\Ensure Prediction region $\hat{C}(\valx)$.

\rlap{\hbox{\textcolor{darkgreen}{\Comment{Sort and map latent samples to the original space:}}}}
\State Compute \(\valy^{(1)}, \dots, \valy^{(M)} \leftarrow\) \Call{Sampling}{$\valx_{\text{test}}, M, f_\phi$}

\rlap{\hbox{\textcolor{darkgreen}{\Comment{Calibrate the radius $\gamma$ using top $K^*$ samples:}}}}

\State Compute the calibrated radius $\gamma =$ \Call{Calibration}{$\{y^{(m)}\}_{m=1}^{K^*}, \mathcal{D}_{\text{cal}}, \alpha$}

\rlap{\hbox{\textcolor{darkgreen}{\Comment{Construct the prediction region $\hat{C}(\valx)$:}}}}
\State Define the prediction region $\hat{C}(\valx)$ as
\[
\hat{C}(\valx_{\text{test}}) = \bigcup_{m=1}^{K^*} B\big(\valy^{(m)}, \gamma\big).
\]
\State \textbf{return} $\hat{C}(\valx_{\text{test}})$.
\end{algorithmic}
\end{algorithm}

\section{Theoretical Analysis}
\label{sec:theoretical_analysis}

In this section, we prove the coverage guarantee for the proposed method.
Our goal is to construct prediction regions $\hat{C}(\rvx)$ such that, for a predetermined coverage level $1 - \alpha$, we have
\[
P(\rvy \in \hat{C}(\rvx)) \geq 1 - \alpha.
\]

The method utilizes a CNF $f_\phi: \mathcal{Y} \times \mathcal{X} \rightarrow \mathcal{Z}$ to transform the original distribution into a standard normal distribution. The prediction region $\hat{C}(\rvx)$ is constructed as a union of balls centered around selected samples in the original space.

\begin{definition}[Exchangeability]
A set of random variables $\{ (\rvx_i, \rvy_i) \}_{i=1}^{N}$ are \emph{exchangeable} if their joint distribution is invariant under any permutation $\pi$, that is,
\begin{align*}
P\left( (\rvx_1, \rvy_1), \ldots, (\rvx_N, \rvy_N) \right) =
P\left( (\rvx_{\pi(1)}, \rvy_{\pi(1)}), \ldots, (\rvx_{\pi(N)}, \rvy_{\pi(N)}) \right),
\end{align*}
for any permutation $\pi$ of $\{1, 2, \ldots, N\}$.
\end{definition}

\begin{lemma}[Exchangeability of Transformed Samples]
\label{lemma:exchangeability}
Assume that the random variables $\{ (\rvx_i, \rvy_i) \}_{i=1}^{N}$ are exchangeable. Then, the transformed random variables $\{ (\rvx_i, \rvz_i) \}_{i=1}^{N}$, where $\rvz_i = f_\phi(\rvy_i, \rvx_i)$, are also exchangeable.
\end{lemma}

\begin{proof}
The transformation $f_\phi$ is a deterministic function applied identically to each sample. Therefore, the exchangeability property is preserved in the transformed space.
\end{proof}

We can now state the coverage guarantee for the proposed VSPS method. The key idea behind this result is that the prediction region we construct for the test point \((\rvx_{\text{test}}, \rvy_{\text{test}})\), using the top \(K^*\) samples as the centers of the prediction ball, has the same coverage probability for all points in the calibration set. This is because the test point \((\rvx_{\text{test}}, \rvy_{\text{test}})\) and the points in the calibration set are exchangeable and the radius \(\gamma\) is determined directly from the calibration set \(\mathcal{D}_{\text{cal}}\). Importantly, the choice of \(K^*\) is independent of \(\mathcal{D}_{\text{cal}}\), as \(K^*\) is selected by optimizing on a separate validation set. 

\begin{theorem}[Coverage Guarantee]
\label{theorem:coverage}
Let $\mathcal{D}_{\text{cal}} = \{(\valx_i, \valy_i)\}_{i=1}^N$ be the calibration set and $(\valx_{\text{test}}, \valy_{\text{test}})$ be the test sample. Suppose that $\mathcal{D}_{\text{cal}} \cup \{(\valx_{\text{test}}, \valy_{\text{test}})\}$ are exchangeable. Then, the prediction region $\hat{C}(\valx_{\text{test}})$ constructed using the optimized parameter $K^*$ satisfies
\[
P(\valy_{\text{test}} \in \hat{C}(\valx_{\text{test}})) \geq 1 - \alpha.
\]
\end{theorem}

\begin{proof}[Proof of Theorem~\ref{theorem:coverage}]
The optimal \(K^*\) is selected using a separate validation set \(\mathcal{D}_{\text{val}}\), which is independent of both the calibration set \(\mathcal{D}_{\text{cal}}\) and the test sample \((\valx_{\text{test}}, \valy_{\text{test}})\). In our method, the radius \(\gamma\) is computed as the \((1-\alpha)(1+1/N)\)-quantile of the non-conformity scores derived from the calibration set.

Define the non-conformity score (\ref{eq:non-conformity}) for each data point in the augmented set, \(\mathcal{D}_{\text{cal}} \cup \{(\valx_{\text{test}}, \valy_{\text{test}})\}
\),
as the distance to the nearest ball center:
\[
d_i = \min_{m=1,\dots, K^*} \|\valy_i - \valy^{(m)}\|_{2}.
\]

Recall that the prediction region for a given test sample \((\valx_{\text{test}}, \valy_{\text{test}})\) is defined as
\[
\hat{C}(\valx_{\text{test}}) = \bigcup_{m=1}^{K^*} B\Bigl(\valy_{\text{test}}^{(m)}, \gamma\Bigr)
=\left\{ \valy \in \mathbb{R}^d : \min_{1\le m \le K^*} d\Bigl(\valy, \valy_{\text{test}}^{(m)}\Bigr) \le \gamma \right\}.
\]

Then, the radius \(\gamma\) is computed as
\[
\gamma = Q_{1-\alpha}\Bigl(\{d_i\}_{i=1}^{N}\Bigr),
\]
where \(Q_{1-\alpha}(\cdot)\) denotes the \((1-\alpha)(1+1/N)\)-quantile over the set of non-conformity scores from the calibration set.

By exchangeability between the calibration samples and the test sample, the non-conformity score for the test sample
\[
d_{\text{test}} = \min_{\valy' \in R_y(\valx_{\text{test}})} d\bigl(\valy_{\text{test}}, \valy'\bigr)
\]
satisfies
\[
P\Bigl(d_{\text{test}} \le Q_{1-\alpha}\bigl(\{d_i\}_{i=1}^{N}\bigr)\Bigr) \ge 1-\alpha.
\]
This implies that
\[
P\Bigl(\valy_{\text{test}} \in \hat{C}(\valx_{\text{test}})\Bigr) \geq 1-\alpha.
\]

Since the selection of \(K^*\) is performed independently using the validation set \(\mathcal{D}_{\text{val}}\), and under the exchangeability assumption, the constructed prediction region \(\hat{C}(\valx_{\text{test}})\) achieves the desired coverage level \(1-\alpha\) for any test sample.
\end{proof}

This theorem establishes that the proposed method provides valid coverage at the desired level $1-\alpha$. The coverage guarantee holds for the test sample despite using parameter optimized on a separate validation set.

\section{Experiments}
\label{sec:experiments}

To evaluate the effectiveness of our proposed \VSPS{} method, we conduct a comprehensive set of experiments on both synthetic and real-world datasets. We compare our approach against existing baseline methods \cite{feldman2023calibrated}, including Naïve Quantile Regression (\text{Naïve QR}), Non-Parametric Directional Quantile Regression (\text{NPDQR}), and Spherically Transformed Directional Quantile Regression (\text{ST-DQR}). Our code and data are publicly available at \href{https://github.com/luo-lorry/VSPS}{GitHub}. 

For both the synthetic and real datasets, the desired coverage is set to $1-\alpha=0.9$. We split each dataset into 38.4\% for training, 25.6\% for calibration, 16\% for validation, and 20\% for testing. For the real datasets both the feature vectors and responses were normalized to zero mean and unit variance. The results are obtained by averaging over 10 random data splits.

For both CVAE with 3-dimensional latent space and the neural network based DQR, the model architecture consisted of 3 hidden layers with 64 units each and leaky ReLU activation. For CNF, we use Conditional Masked Autoregressive Flow (MAF) implemented by stacking conditional MADE layers \cite{papamakarios2017masked}. We used the Adam optimizer with learning rate $10^{-3}$ and batch size 256, employing early stopping during training with the validation data.

We evaluate our results based on three key metrics: 

\noindent
\textit{Marginal Coverage} (Coverage): The proportion of test samples for which the correct multi-target output is included in the prediction region.

\noindent
\textit{Prediction Region Size} (Size): The size of the prediction region, calculated by discretizing the target output space $\mathcal{Y}$ into a grid and counting the number of grid points contained within the region.

\noindent
\textit{Conditional Coverage} (Cond. Coverage): The conditional coverage, which measures the coverage conditioned on specific values of the input features.

\subsection{Synthetic Data}

Measuring conditional coverage is challenging  because the conditional probability $p_{\rvy|\rvx}(y|x)$ is unknown.To address this, we use a synthetic dataset with a v-shaped structure and 2-dimensional outputs. The 1-dimensional feature $X \in \{1.5, 2.0, 2.5\}$ is discrete, allowing us to explicitly compute conditional coverage by sampling multiple pairs $(x, Y_i)$ at each value of $x$. The synthetic data generation process follows \cite{feldman2023calibrated}.

For the experiments, we report the conditional coverage as the smallest conditional coverage across the values of $X \in \{1.5, 2.0, 2.5\}$. This approach rigorously tests our method’s ability to maintain coverage across different conditional distributions. The results for Marginal Coverage, Prediction Region Size, and Conditional Coverage on the synthetic dataset are provided in Table~\ref{tab:synthetic_data_results}.

\begin{table}[t]
\centering
\small
\begin{adjustbox}{max width=\columnwidth}
\begin{tabular}{@{}lccc@{}}
\toprule
\textbf{Metric} & \text{Naïve QR} & \text{NPDQR} & \text{VSPS} \\
\midrule
Coverage & 90.16 (0.16) & 90.44 (1.26) & 90.06 (1.30) \\
Size & 1266.68 (18.34) & 1367.86 (90.80) & \textbf{104.34 (2.43)} \\
Cond. Coverage & 86.08 (0.93) & 84.56 (1.31) & \textbf{87.57 (1.55)} \\
\bottomrule
\end{tabular}
\end{adjustbox}
\caption{Results for Marginal Coverage, Prediction Region Size, and Conditional Coverage for the synthetic dataset.}
\label{tab:synthetic_data_results}
\end{table}

\subsection{Real Data}

We further evaluate our method on several real-world benchmark datasets: Blog Feedback (Blog), Physicochemical Properties of Protein Tertiary Structure (Bio), House Sales in King County, USA (House), and Medical Expenditure Panel Survey datasets from 2019, 2020, and 2021 (meps\_19, meps\_20, meps\_21). These datasets are widely used as benchmarks in prior works \cite{romano2019conformalized, romano2020classification}. We modify each dataset to have a 2-dimensional response following the approach described in \cite{feldman2023calibrated}. The results from these datasets are summarized in Table~\ref{tab:real_data_experiments}.

\begin{table}[ht]
\centering
\begin{adjustbox}{max width=\columnwidth}
\begin{tabular}{llcccc}
\toprule
\multirow{2}{*}{\textbf{Dataset}} & \multirow{2}{*}{\textbf{Metric}} & \multicolumn{4}{c}{\textbf{Method}} \\
\cmidrule(lr){3-6}
 &  & \text{Naïve QR} & \text{NPDQR} & \text{ST-DQR} & \text{VSPS} \\
\midrule

\multirow{2}{*}{bio} 
    & Coverage & 90.06 (0.38) & 90.07 (0.40) & 90.08 (0.45) & 89.86 (0.36) \\
    & Size     & 418.17 (19.79) & 562.45 (126.28) & 324.53 (29.10) & \textbf{137.39 (15.63)} \\

\multirow{2}{*}{blog\_data} 
    & Coverage & 90.03 (0.53) & 89.85 (0.45) & 90.09 (0.43) & 90.18 (0.32) \\
    & Size     & 217.39 (10.60) & 547.27 (467.80) & 473.71 (95.68) & \textbf{99.84 (23.44)} \\

\multirow{2}{*}{house} 
    & Coverage & 90.28 (0.68) & 89.85 (0.75) & 90.02 (0.80) & 89.84 (0.68) \\
    & Size     & 448.27 (18.10) & 618.24 (259.47) & 334.61 (24.08) & \textbf{165.19 (8.00)} \\

\multirow{2}{*}{meps\_19} 
    & Coverage & 90.04 (0.56) & 90.11 (0.71) & 90.33 (0.42) & 89.91 (0.73) \\
    & Size     & 348.58 (83.96) & 879.36 (874.52) & 169.99 (28.86) & \textbf{120.44 (26.40)} \\

\multirow{2}{*}{meps\_20} 
    & Coverage & 90.40 (0.62) & 89.91 (0.72) & 90.15 (1.01) & 89.89 (0.54) \\
    & Size     & 366.75 (25.34) & 929.51 (519.18) & 181.33 (15.45) & \textbf{121.92 (13.53)} \\

\multirow{2}{*}{meps\_21} 
    & Coverage & 89.83 (0.44) & 89.66 (0.42) & 89.95 (0.77) & 89.77 (0.68) \\
    & Size     & 369.72 (57.84) & 721.05 (498.28) & 182.67 (24.56) & \textbf{116.71 (19.90)} \\

\bottomrule
\end{tabular}
\end{adjustbox}
\caption{Results for Marginal Coverage and Prediction Region Size across the real datasets.}
\label{tab:real_data_experiments}
\end{table}

\begin{figure}
  \centering
  \begin{tabular}{ccc}
    {\includegraphics[width=0.3\columnwidth]{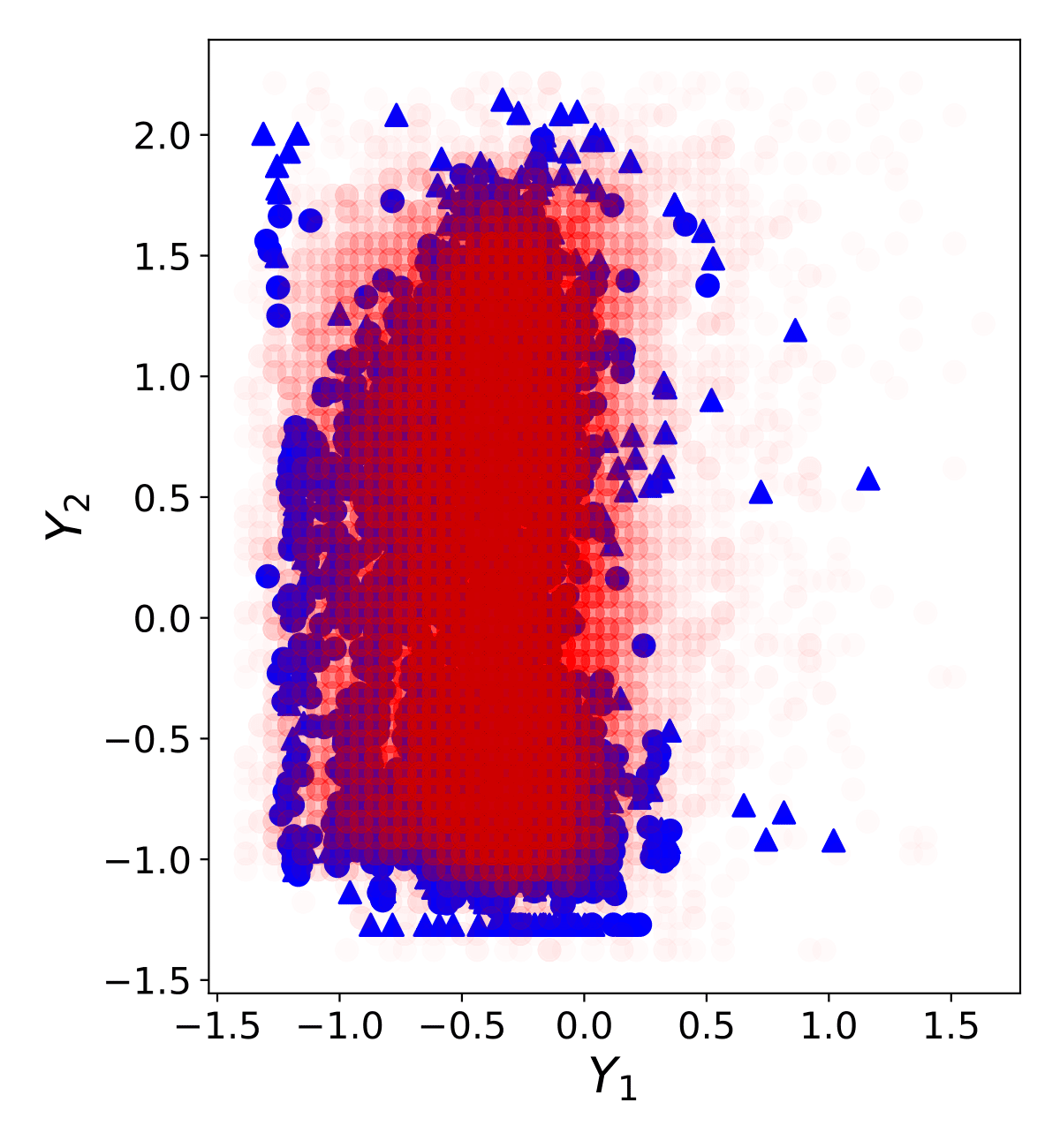}} & 
    {\includegraphics[width=0.3\columnwidth]{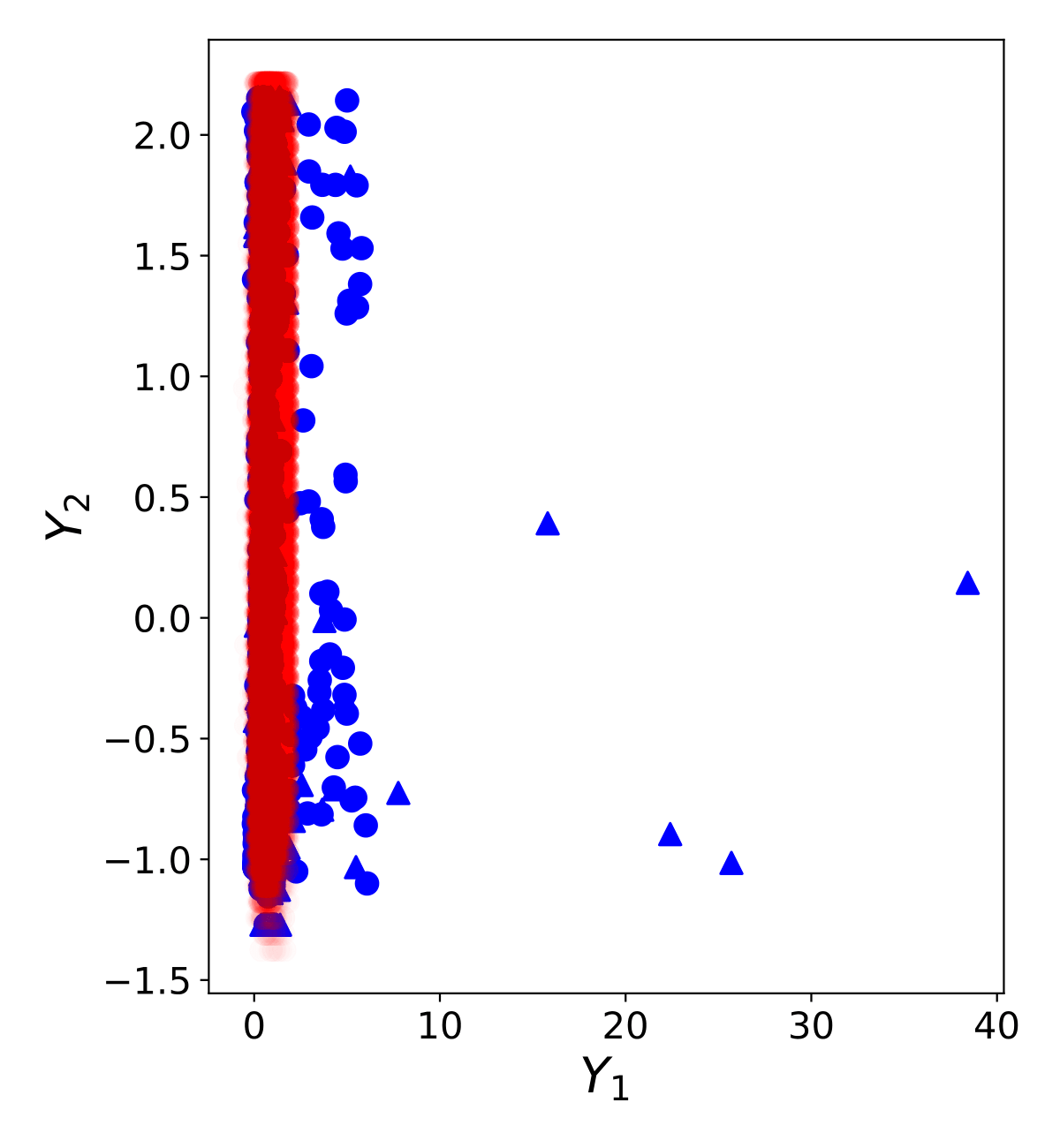}} & 
    {\includegraphics[width=0.3\columnwidth]{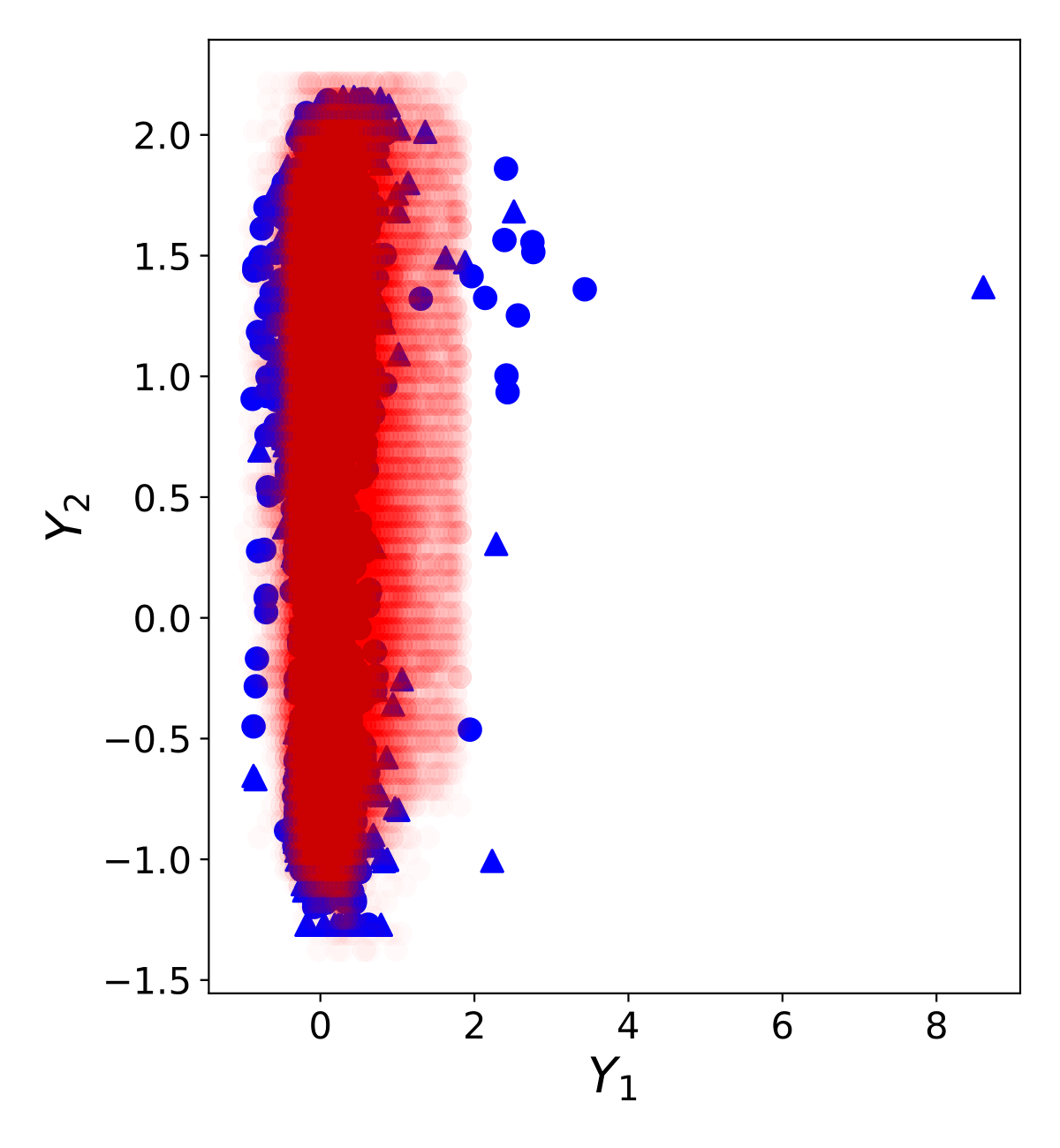}} \\
    
    {\includegraphics[width=0.3\columnwidth]{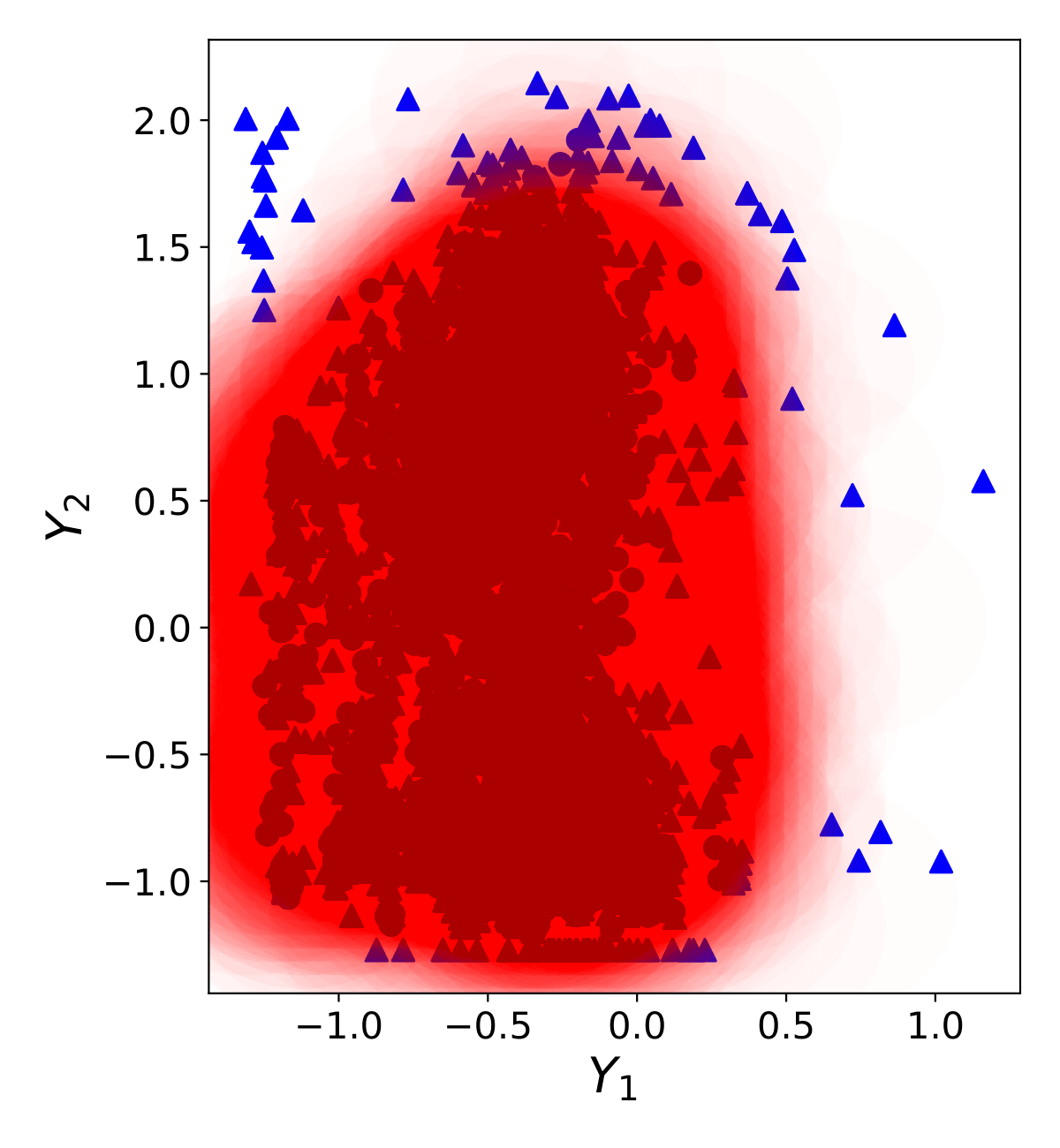}} & 
    {\includegraphics[width=0.3\columnwidth]{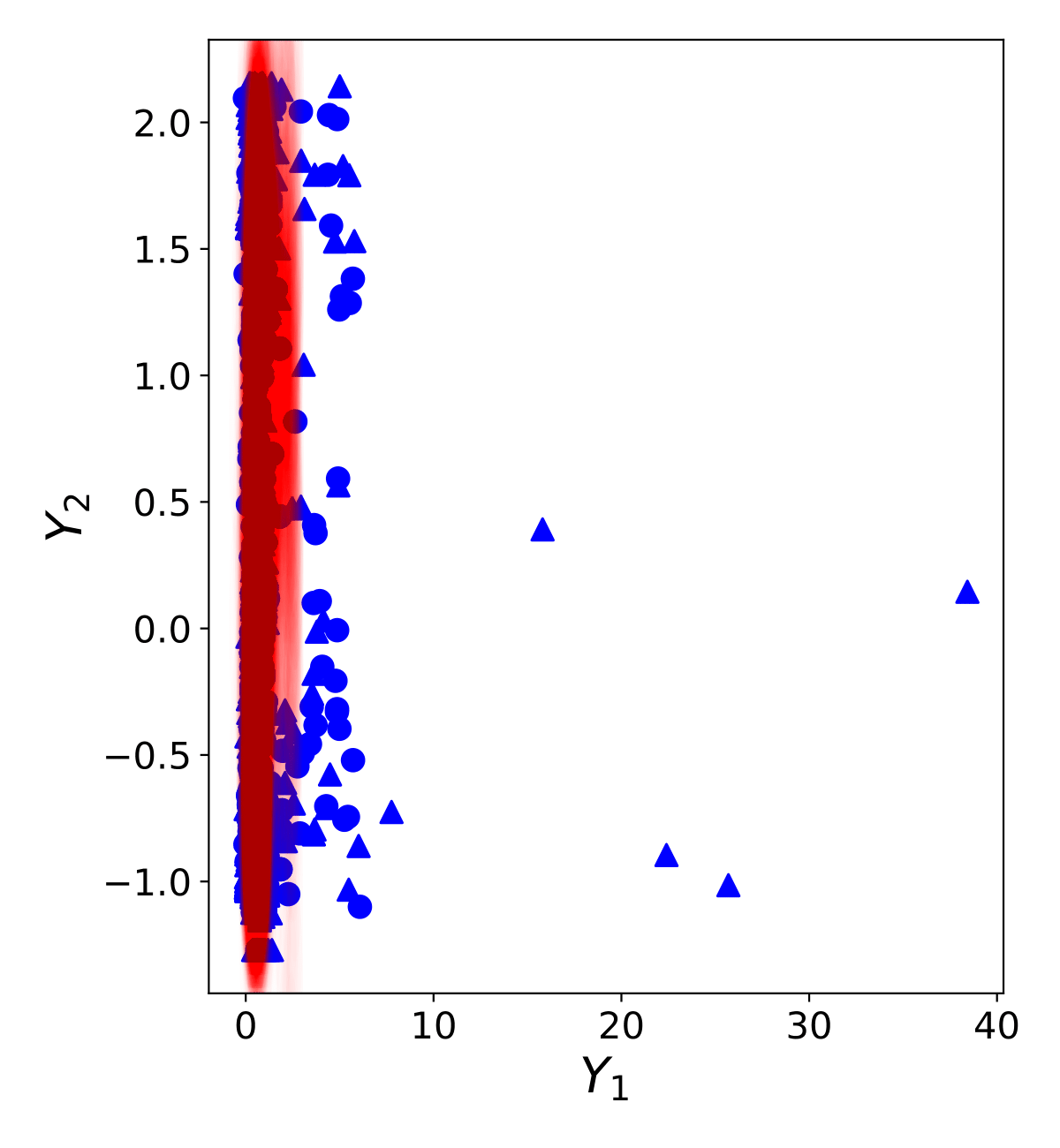}} & 
    {\includegraphics[width=0.3\columnwidth]{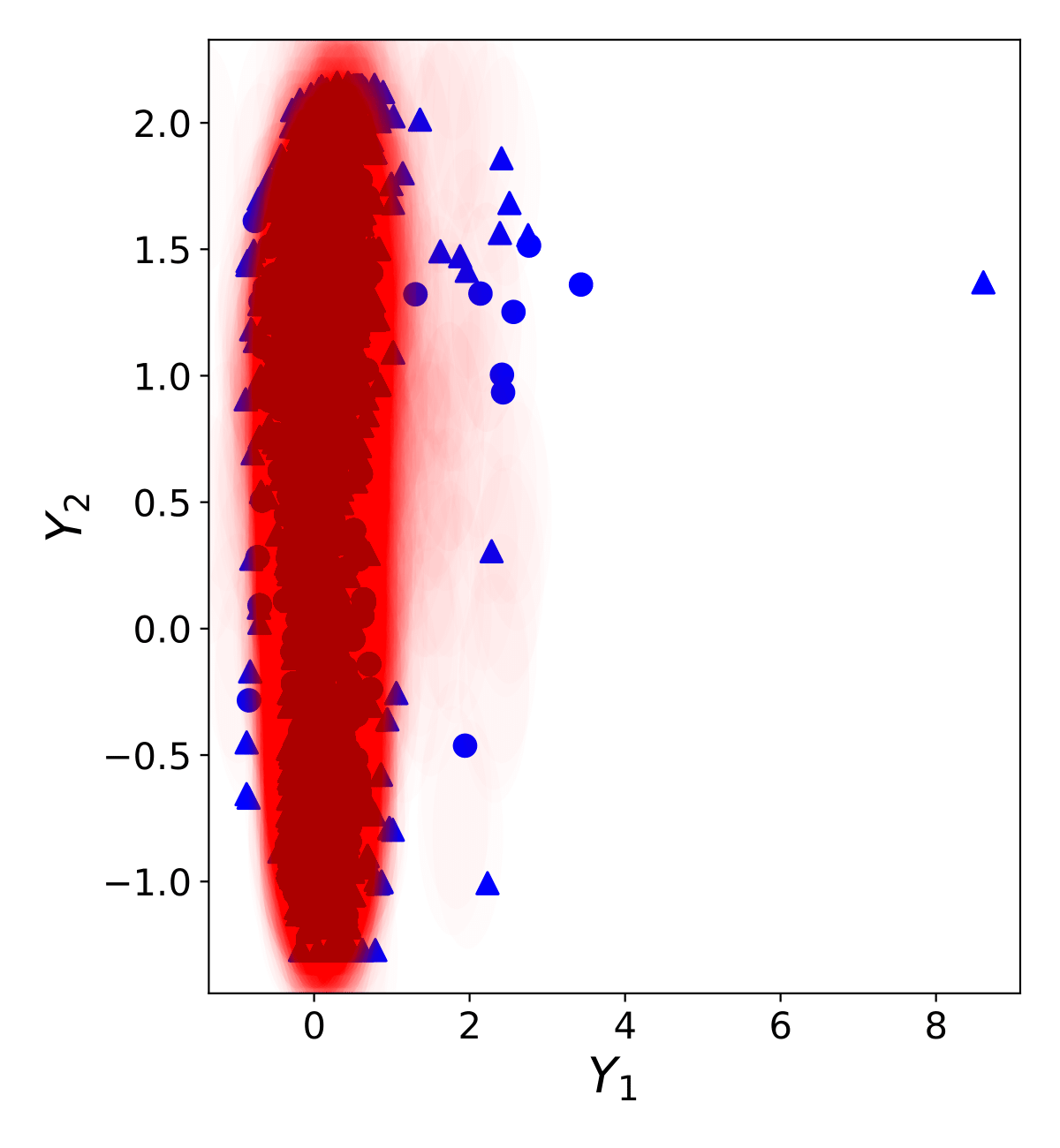}} \\
    
    {\includegraphics[width=0.3\columnwidth]{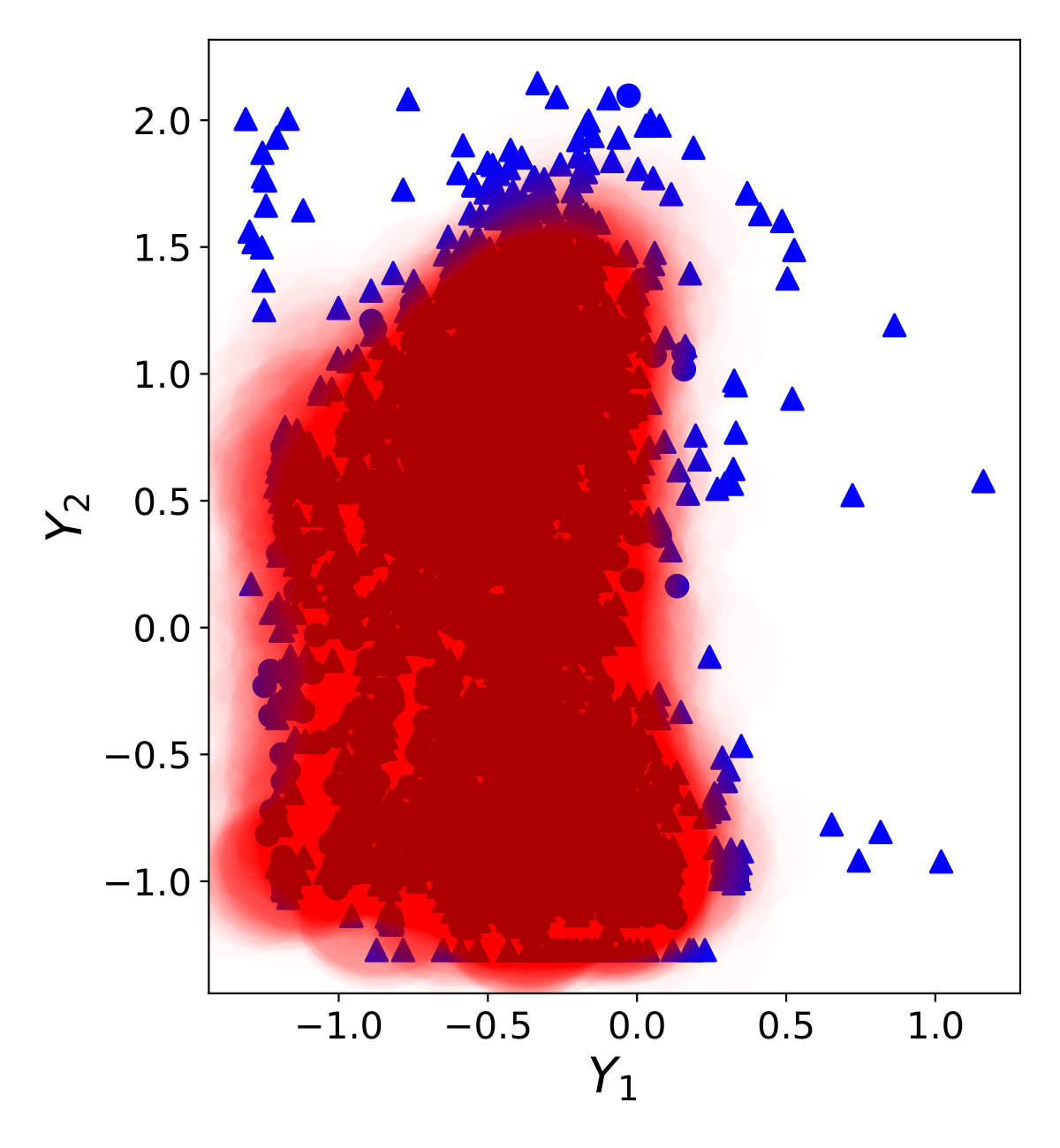}} & 
    {\includegraphics[width=0.3\columnwidth]{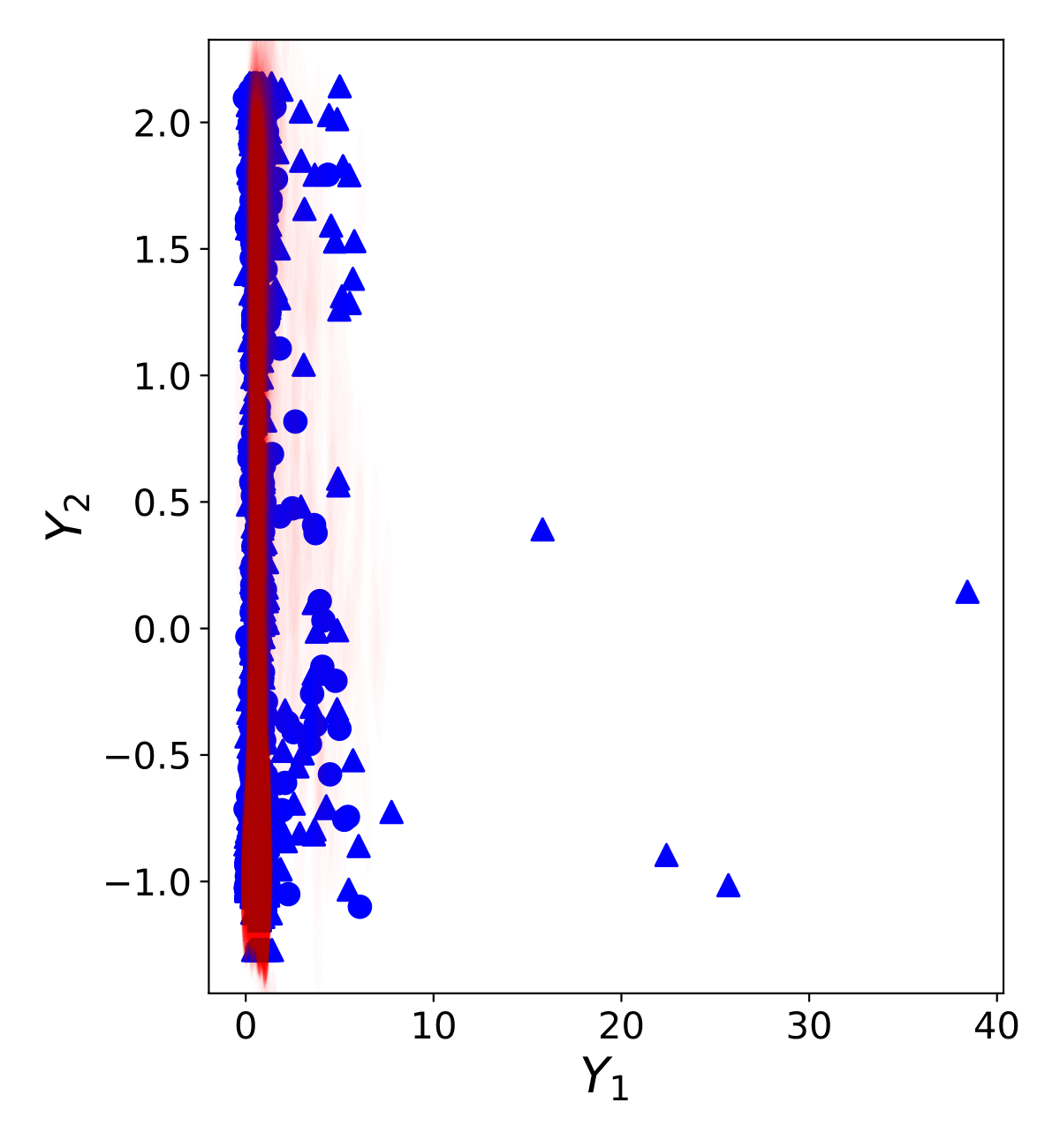}} & 
    {\includegraphics[width=0.3\columnwidth]{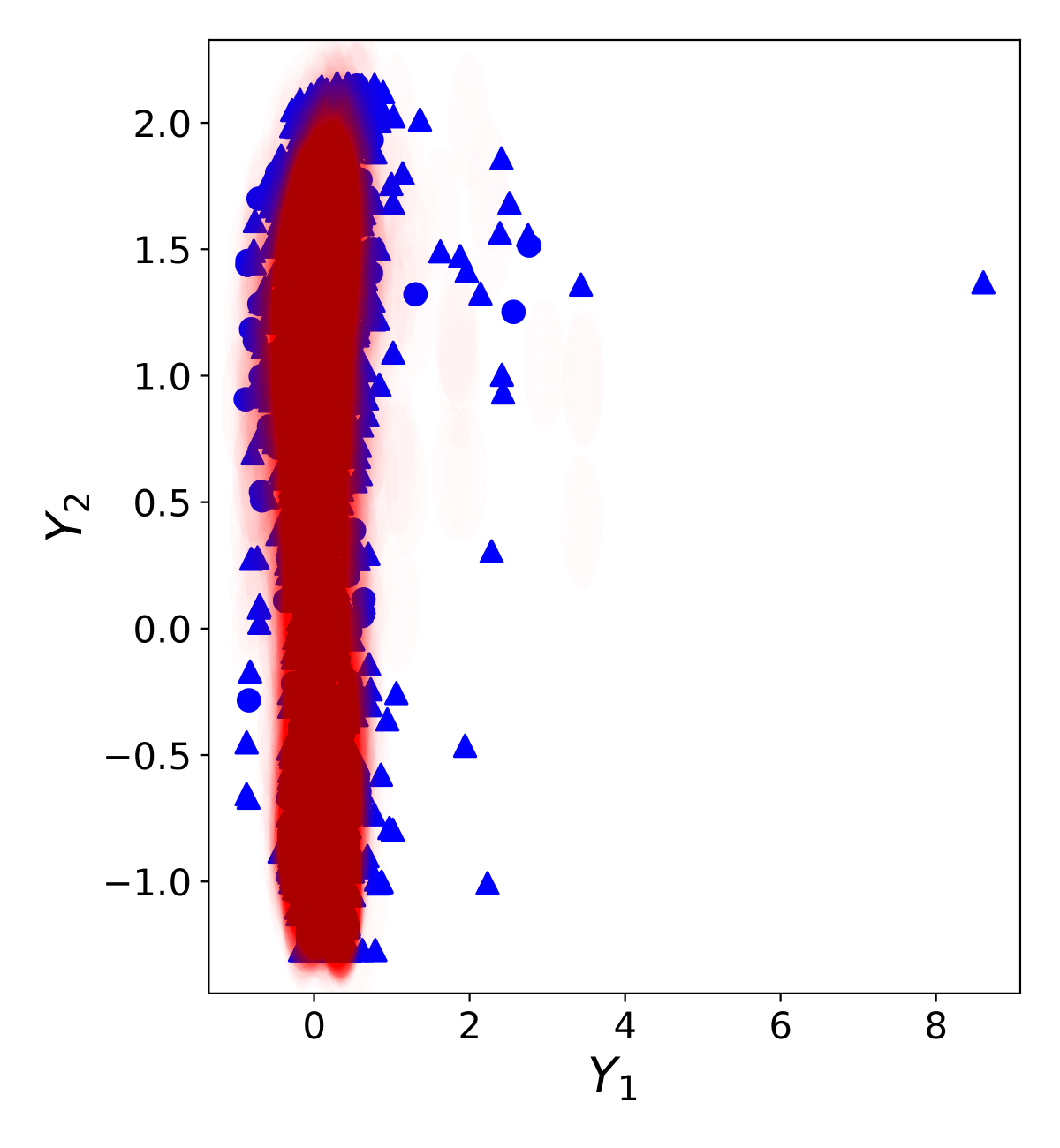}} \\
    
    {\includegraphics[width=0.3\columnwidth]{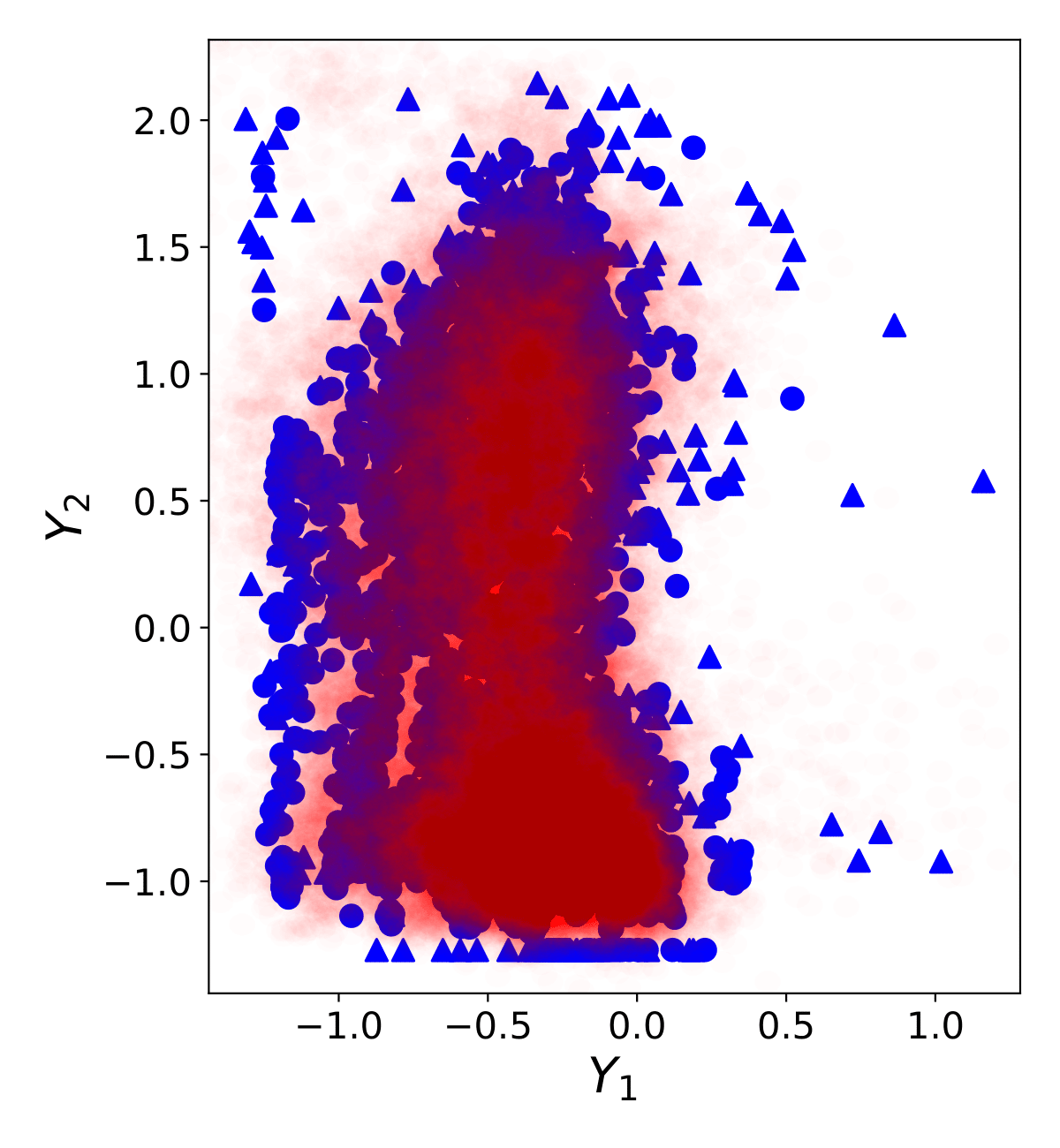}} & 
    {\includegraphics[width=0.3\columnwidth]{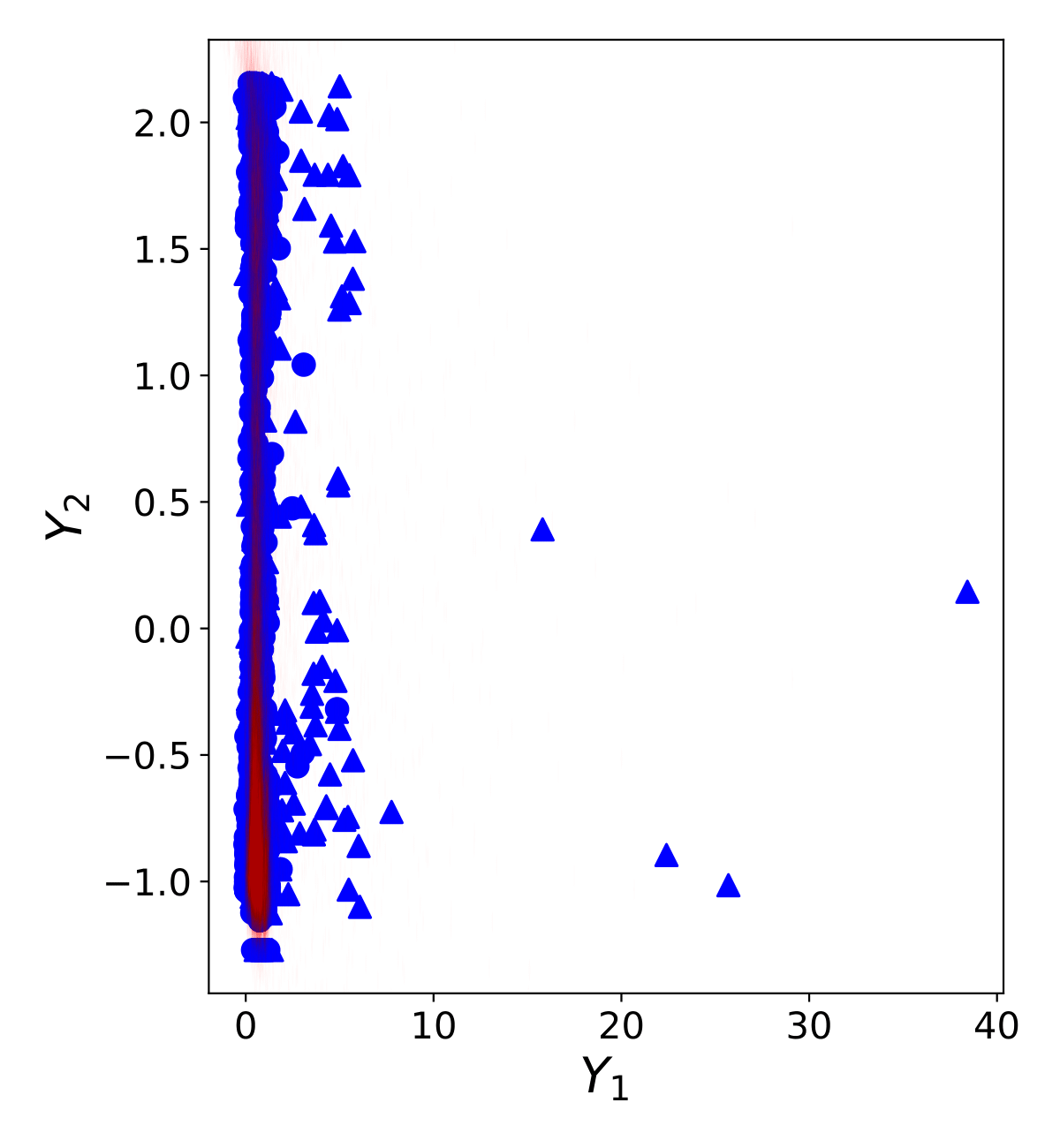}} & 
    {\includegraphics[width=0.3\columnwidth]{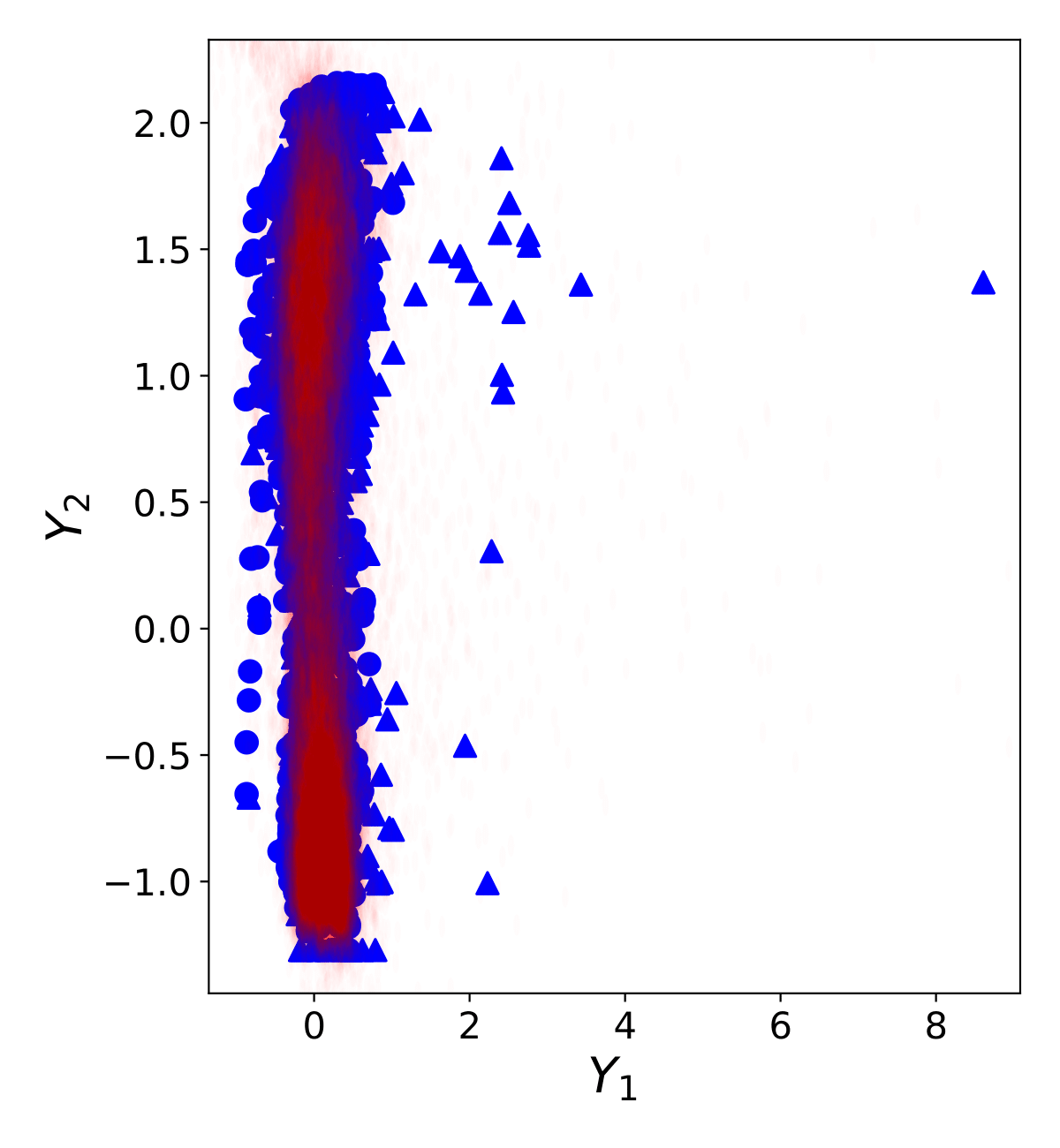}} \\
  \end{tabular}%
  
  \caption{Comparison of prediction regions for the Bio dataset using four methods: \text{Naïve QR}, \text{NPDQR}, \text{ST-DQR}, and our proposed \text{VSPS}. The dataset is partitioned into three clusters using K-Means. Circles represent covered points, while triangles indicate miscovered points. Each row corresponds to a different method, with \text{VSPS} demonstrating the most compact prediction regions across all clusters while maintaining the $1-\alpha$ coverage. The Bio dataset comprises two protein structural features ($Y_0$ and $Y_1$) as response variables and an 8-dimensional protein feature vector ($X$) as predictors.}

  \label{fig:bio_data_results}
\end{figure}

Additionally, Figure~\ref{fig:bio_data_results} visualizes the predictive regions for select datasets, illustrating that \VSPS{} captures the underlying distribution more effectively with smaller regions compared to \text{ST-DQR}, \text{NPDQR}, and \text{Naïve QR}.

The proposed \VSPS{} method demonstrates optimal efficiency and conditional coverage. This superior performance can be attributed to the adaptive nature of \VSPS{}, which allows it to adjust the number of selected samples and the ball radius based on local data characteristics. The use of CNFs enables accurate modeling of complex conditional distributions, while the calibration procedure ensures valid coverage while minimizing prediction set size.

\section{Conclusion}
\label{sec:conclusion}

In this work, we introduced the Volume-Sorted Prediction Set (\VSPS{}) method, a novel approach for constructing non-parametric and flexible quantile regions of arbitrary shapes in multi-target regression tasks. Leveraging conditional normalizing flows, \VSPS{} captures complex dependencies between input features and multiple output variables, enabling more informative and efficient prediction sets. Our key contributions include a novel algorithm for identifying high-density areas, a conformal calibration procedure guaranteeing pre-specified coverage levels, and extensive empirical evaluation demonstrating \VSPS{}'s superiority over existing techniques such as ST-DQR, NPDQR, and Naïve QR. Experiments on both synthetic and real-world datasets showcased \VSPS{}'s ability to generate informative and efficient quantile regions for multi-target data, with particularly significant improvements in scenarios involving complex, non-linear relationships.


\begin{thebibliography}{54}
\expandafter\ifx\csname natexlab\endcsname\relax\def\natexlab#1{#1}\fi
\providecommand{\url}[1]{\texttt{#1}}
\providecommand{\href}[2]{#2}
\providecommand{\path}[1]{#1}
\providecommand{\DOIprefix}{doi:}
\providecommand{\ArXivprefix}{arXiv:}
\providecommand{\URLprefix}{URL: }
\providecommand{\Pubmedprefix}{pmid:}
\providecommand{\doi}[1]{\href{http://dx.doi.org/#1}{\path{#1}}}
\providecommand{\Pubmed}[1]{\href{pmid:#1}{\path{#1}}}
\providecommand{\bibinfo}[2]{#2}
\ifx\xfnm\relax \def\xfnm[#1]{\unskip,\space#1}\fi
\bibitem[{Gon{\c{c}}alves et~al.(2023)Gon{\c{c}}alves, Pedronette, and da~Silva~Torres}]{gonccalves2023regression}
\bibinfo{author}{F.~M.~F. Gon{\c{c}}alves}, \bibinfo{author}{D.~C.~G. Pedronette}, \bibinfo{author}{R.~da~Silva~Torres},
\newblock \bibinfo{title}{Regression by re-ranking},
\newblock \bibinfo{journal}{Pattern Recognition} \bibinfo{volume}{140} (\bibinfo{year}{2023}) \bibinfo{pages}{109577}.
\bibitem[{Patel et~al.(2024)Patel, Rayan, and Tewari}]{patel2024conformal}
\bibinfo{author}{Y.~P. Patel}, \bibinfo{author}{S.~Rayan}, \bibinfo{author}{A.~Tewari},
\newblock \bibinfo{title}{Conformal contextual robust optimization},
\newblock in: \bibinfo{booktitle}{International Conference on Artificial Intelligence and Statistics}, \bibinfo{organization}{PMLR}, \bibinfo{year}{2024}, pp. \bibinfo{pages}{2485--2493}.
\bibitem[{Dheur et~al.(2025)Dheur, Fontana, Estievenart, Desobry, and Taieb}]{dheur2025multi}
\bibinfo{author}{V.~Dheur}, \bibinfo{author}{M.~Fontana}, \bibinfo{author}{Y.~Estievenart}, \bibinfo{author}{N.~Desobry}, \bibinfo{author}{S.~B. Taieb},
\newblock \bibinfo{title}{Multi-output conformal regression: A unified comparative study with new conformity scores},
\newblock \bibinfo{journal}{arXiv preprint arXiv:2501.10533}  (\bibinfo{year}{2025}).
\bibitem[{Schuell et~al.(2005)Schuell, Gruenberger, Kornek, Dworan, Depisch, Lang, Schneeweiss, and Scheithauer}]{schuell2005side}
\bibinfo{author}{B.~Schuell}, \bibinfo{author}{T.~Gruenberger}, \bibinfo{author}{G.~Kornek}, \bibinfo{author}{N.~Dworan}, \bibinfo{author}{D.~Depisch}, \bibinfo{author}{F.~Lang}, \bibinfo{author}{B.~Schneeweiss}, \bibinfo{author}{W.~Scheithauer},
\newblock \bibinfo{title}{Side effects during chemotherapy predict tumour response in advanced colorectal cancer},
\newblock \bibinfo{journal}{British journal of cancer} \bibinfo{volume}{93} (\bibinfo{year}{2005}) \bibinfo{pages}{744--748}.
\bibitem[{Nakano et~al.(2022)Nakano, Pliakos, and Vens}]{nakano2022deep}
\bibinfo{author}{F.~K. Nakano}, \bibinfo{author}{K.~Pliakos}, \bibinfo{author}{C.~Vens},
\newblock \bibinfo{title}{Deep tree-ensembles for multi-output prediction},
\newblock \bibinfo{journal}{Pattern Recognition} \bibinfo{volume}{121} (\bibinfo{year}{2022}) \bibinfo{pages}{108211}.
\bibitem[{Hua et~al.(2025)Hua, Feng, Song, Wu, and Kittler}]{hua2025mmdg}
\bibinfo{author}{Y.~Hua}, \bibinfo{author}{Z.~Feng}, \bibinfo{author}{X.~Song}, \bibinfo{author}{X.-J. Wu}, \bibinfo{author}{J.~Kittler},
\newblock \bibinfo{title}{Mmdg-dti: Drug--target interaction prediction via multimodal feature fusion and domain generalization},
\newblock \bibinfo{journal}{Pattern Recognition} \bibinfo{volume}{157} (\bibinfo{year}{2025}) \bibinfo{pages}{110887}.
\bibitem[{Koenker and Bassett~Jr(1978)}]{koenker1978regression}
\bibinfo{author}{R.~Koenker}, \bibinfo{author}{G.~Bassett~Jr},
\newblock \bibinfo{title}{Regression quantiles},
\newblock \bibinfo{journal}{Econometrica: journal of the Econometric Society}  (\bibinfo{year}{1978}) \bibinfo{pages}{33--50}.
\bibitem[{Izbicki et~al.(2020)Izbicki, Shimizu, and Stern}]{izbicki2019flexible}
\bibinfo{author}{R.~Izbicki}, \bibinfo{author}{G.~Shimizu}, \bibinfo{author}{R.~Stern},
\newblock \bibinfo{title}{Flexible distribution-free conditional predictive bands using density estimators},
\newblock in: \bibinfo{booktitle}{International Conference on Artificial Intelligence and Statistics}, \bibinfo{organization}{PMLR}, \bibinfo{year}{2020}, pp. \bibinfo{pages}{3068--3077}.
\bibitem[{Gupta et~al.(2022)Gupta, Kuchibhotla, and Ramdas}]{gupta2022nested}
\bibinfo{author}{C.~Gupta}, \bibinfo{author}{A.~K. Kuchibhotla}, \bibinfo{author}{A.~Ramdas},
\newblock \bibinfo{title}{Nested conformal prediction and quantile out-of-bag ensemble methods},
\newblock \bibinfo{journal}{Pattern Recognition} \bibinfo{volume}{127} (\bibinfo{year}{2022}) \bibinfo{pages}{108496}.
\bibitem[{Bo{\v{c}}ek and {\v{S}}iman(2017)}]{bovcek2017directional}
\bibinfo{author}{P.~Bo{\v{c}}ek}, \bibinfo{author}{M.~{\v{S}}iman},
\newblock \bibinfo{title}{Directional quantile regression in r},
\newblock \bibinfo{journal}{Kybernetika} \bibinfo{volume}{53} (\bibinfo{year}{2017}) \bibinfo{pages}{480--492}.
\bibitem[{Carlier et~al.(2016)Carlier, Chernozhukov, and Galichon}]{carlier2016vector}
\bibinfo{author}{G.~Carlier}, \bibinfo{author}{V.~Chernozhukov}, \bibinfo{author}{A.~Galichon},
\newblock \bibinfo{title}{Vector quantile regression: an optimal transport approach},
\newblock \bibinfo{journal}{The Annals of Statistics} \bibinfo{volume}{44} (\bibinfo{year}{2016}) \bibinfo{pages}{1165--1192}.
\bibitem[{Feldman et~al.(2023)Feldman, Bates, and Romano}]{feldman2023calibrated}
\bibinfo{author}{S.~Feldman}, \bibinfo{author}{S.~Bates}, \bibinfo{author}{Y.~Romano},
\newblock \bibinfo{title}{Calibrated multiple-output quantile regression with representation learning},
\newblock \bibinfo{journal}{Journal of Machine Learning Research} \bibinfo{volume}{24} (\bibinfo{year}{2023}) \bibinfo{pages}{1--48}.
\bibitem[{Vovk et~al.(2005)Vovk, Gammerman, and Shafer}]{vovk2005algorithmic}
\bibinfo{author}{V.~Vovk}, \bibinfo{author}{A.~Gammerman}, \bibinfo{author}{G.~Shafer}, \bibinfo{title}{Algorithmic learning in a random world}, volume~\bibinfo{volume}{29}, \bibinfo{publisher}{Springer}, \bibinfo{year}{2005}.
\bibitem[{Manokhin(2022)}]{manokhin2022awesome}
\bibinfo{author}{V.~Manokhin},
\newblock \bibinfo{title}{Awesome conformal prediction},
\newblock \bibinfo{journal}{If you use Awesome Conformal Prediction. please cite it as below}  (\bibinfo{year}{2022}).
\bibitem[{Lei and Wasserman(2014)}]{lei2014distribution}
\bibinfo{author}{J.~Lei}, \bibinfo{author}{L.~Wasserman},
\newblock \bibinfo{title}{Distribution-free prediction bands for non-parametric regression},
\newblock \bibinfo{journal}{Journal of the Royal Statistical Society Series B: Statistical Methodology} \bibinfo{volume}{76} (\bibinfo{year}{2014}) \bibinfo{pages}{71--96}.
\bibitem[{Romano et~al.(2019)Romano, Patterson, and Candes}]{romano2019conformalized}
\bibinfo{author}{Y.~Romano}, \bibinfo{author}{E.~Patterson}, \bibinfo{author}{E.~Candes},
\newblock \bibinfo{title}{Conformalized quantile regression},
\newblock \bibinfo{journal}{Advances in neural information processing systems} \bibinfo{volume}{32} (\bibinfo{year}{2019}).
\bibitem[{Luo and Zhou(2024)}]{luo2024conformal}
\bibinfo{author}{R.~Luo}, \bibinfo{author}{Z.~Zhou},
\newblock \bibinfo{title}{Conformal thresholded intervals for efficient regression},
\newblock \bibinfo{journal}{arXiv preprint arXiv:2407.14495}  (\bibinfo{year}{2024}).
\bibitem[{Romano et~al.(2020)Romano, Sesia, and Candes}]{romano2020classification}
\bibinfo{author}{Y.~Romano}, \bibinfo{author}{M.~Sesia}, \bibinfo{author}{E.~Candes},
\newblock \bibinfo{title}{Classification with valid and adaptive coverage},
\newblock \bibinfo{journal}{Advances in Neural Information Processing Systems} \bibinfo{volume}{33} (\bibinfo{year}{2020}) \bibinfo{pages}{3581--3591}.
\bibitem[{Luo et~al.(2024)Luo, Bao, Zhou, and Dang}]{luo2024game}
\bibinfo{author}{R.~Luo}, \bibinfo{author}{J.~Bao}, \bibinfo{author}{Z.~Zhou}, \bibinfo{author}{C.~Dang},
\newblock \bibinfo{title}{Game-theoretic defenses for robust conformal prediction against adversarial attacks in medical imaging},
\newblock \bibinfo{journal}{arXiv preprint arXiv:2411.04376}  (\bibinfo{year}{2024}).
\bibitem[{Liu et~al.(2024)Liu, Zeng, Huang, Zhuang, Vong, and Wei}]{liu2024c}
\bibinfo{author}{K.~Liu}, \bibinfo{author}{H.~Zeng}, \bibinfo{author}{J.~Huang}, \bibinfo{author}{H.~Zhuang}, \bibinfo{author}{C.-M. Vong}, \bibinfo{author}{H.~Wei},
\newblock \bibinfo{title}{C-adapter: Adapting deep classifiers for efficient conformal prediction sets},
\newblock \bibinfo{journal}{arXiv preprint arXiv:2410.09408}  (\bibinfo{year}{2024}).
\bibitem[{Luo and Colombo(2024)}]{luo2024entropy}
\bibinfo{author}{R.~Luo}, \bibinfo{author}{N.~Colombo},
\newblock \bibinfo{title}{Entropy reweighted conformal classification},
\newblock \bibinfo{journal}{Proceedings of Machine Learning Research} \bibinfo{volume}{230} (\bibinfo{year}{2024}) \bibinfo{pages}{1--13}.
\bibitem[{Huang et~al.(2024)Huang, Xi, Zhang, Yao, Qiu, and Wei}]{huang2024conformal}
\bibinfo{author}{J.~Huang}, \bibinfo{author}{H.~Xi}, \bibinfo{author}{L.~Zhang}, \bibinfo{author}{H.~Yao}, \bibinfo{author}{Y.~Qiu}, \bibinfo{author}{H.~Wei},
\newblock \bibinfo{title}{Conformal prediction for deep classifier via label ranking},
\newblock in: \bibinfo{booktitle}{Proceedings of the 41st International Conference on Machine Learning}, \bibinfo{year}{2024}, pp. \bibinfo{pages}{20331--20347}.
\bibitem[{Luo and Zhou(2024)}]{luo2024trustworthy}
\bibinfo{author}{R.~Luo}, \bibinfo{author}{Z.~Zhou},
\newblock \bibinfo{title}{Trustworthy classification through rank-based conformal prediction sets},
\newblock \bibinfo{journal}{arXiv preprint arXiv:2407.04407}  (\bibinfo{year}{2024}).
\bibitem[{Huang et~al.(2025)Huang, Cai, Liu, Cao, Wei, and An}]{huang2025conformal}
\bibinfo{author}{J.~Huang}, \bibinfo{author}{X.~Cai}, \bibinfo{author}{K.~Liu}, \bibinfo{author}{Y.~Cao}, \bibinfo{author}{H.~Wei}, \bibinfo{author}{B.~An}, \bibinfo{title}{Conformal prediction for deep classifier via truncating}, \bibinfo{year}{2025}. \URLprefix \url{https://openreview.net/forum?id=uUkpYafkVl}.
\bibitem[{Zeng et~al.(2025)Zeng, Liu, Jing, and Wei}]{zeng2025parametric}
\bibinfo{author}{H.~Zeng}, \bibinfo{author}{K.~Liu}, \bibinfo{author}{B.~Jing}, \bibinfo{author}{H.~Wei},
\newblock \bibinfo{title}{Parametric scaling law of tuning bias in conformal prediction},
\newblock \bibinfo{journal}{arXiv preprint arXiv:2502.03023}  (\bibinfo{year}{2025}).
\bibitem[{Tawachi and Laufer-Goldshtein(2025)}]{tawachi2025multidimensional}
\bibinfo{author}{Y.~Tawachi}, \bibinfo{author}{B.~Laufer-Goldshtein},
\newblock \bibinfo{title}{Multi-dimensional conformal prediction},
\newblock in: \bibinfo{booktitle}{The Thirteenth International Conference on Learning Representations}, \bibinfo{year}{2025}.
\bibitem[{Luo et~al.(2023)Luo, Nettasinghe, and Krishnamurthy}]{luo2023anomalous}
\bibinfo{author}{R.~Luo}, \bibinfo{author}{B.~Nettasinghe}, \bibinfo{author}{V.~Krishnamurthy},
\newblock \bibinfo{title}{Anomalous edge detection in edge exchangeable social network models},
\newblock in: \bibinfo{booktitle}{Conformal and probabilistic prediction with applications}, \bibinfo{organization}{PMLR}, \bibinfo{year}{2023}, pp. \bibinfo{pages}{287--310}.
\bibitem[{Severo et~al.(2023)Severo, Townsend, Khisti, and Makhzani}]{severo2023one}
\bibinfo{author}{D.~Severo}, \bibinfo{author}{J.~Townsend}, \bibinfo{author}{A.~J. Khisti}, \bibinfo{author}{A.~Makhzani},
\newblock \bibinfo{title}{One-shot compression of large edge-exchangeable graphs using bits-back coding},
\newblock in: \bibinfo{booktitle}{International Conference on Machine Learning}, \bibinfo{organization}{PMLR}, \bibinfo{year}{2023}, pp. \bibinfo{pages}{30633--30645}.
\bibitem[{Lunde et~al.(2023)Lunde, Levina, and Zhu}]{lunde2023conformal}
\bibinfo{author}{R.~Lunde}, \bibinfo{author}{E.~Levina}, \bibinfo{author}{J.~Zhu},
\newblock \bibinfo{title}{Conformal prediction for network-assisted regression},
\newblock \bibinfo{journal}{arXiv preprint arXiv:2302.10095}  (\bibinfo{year}{2023}).
\bibitem[{Luo and Zhou(2024)}]{luo2024conformalized}
\bibinfo{author}{R.~Luo}, \bibinfo{author}{Z.~Zhou},
\newblock \bibinfo{title}{Conformalized interval arithmetic with symmetric calibration},
\newblock \bibinfo{journal}{arXiv preprint arXiv:2408.10939}  (\bibinfo{year}{2024}).
\bibitem[{Marandon(2024)}]{marandon2024conformal}
\bibinfo{author}{A.~Marandon},
\newblock \bibinfo{title}{Conformal link prediction for false discovery rate control},
\newblock \bibinfo{journal}{TEST}  (\bibinfo{year}{2024}) \bibinfo{pages}{1--22}.
\bibitem[{Luo and Colombo(2025)}]{luo2025conformal}
\bibinfo{author}{R.~Luo}, \bibinfo{author}{N.~Colombo},
\newblock \bibinfo{title}{Conformal load prediction with transductive graph autoencoders},
\newblock \bibinfo{journal}{Machine Learning} \bibinfo{volume}{114} (\bibinfo{year}{2025}) \bibinfo{pages}{1--22}.
\bibitem[{Wang et~al.(2025)Wang, Zhou, and Luo}]{wang2025enhancing}
\bibinfo{author}{T.~Wang}, \bibinfo{author}{Z.~Zhou}, \bibinfo{author}{R.~Luo},
\newblock \bibinfo{title}{Enhancing trustworthiness of graph neural networks with rank-based conformal training},
\newblock \bibinfo{journal}{arXiv preprint arXiv:2501.02767}  (\bibinfo{year}{2025}).
\bibitem[{Messoudi et~al.(2021)Messoudi, Destercke, and Rousseau}]{messoudi2021copula}
\bibinfo{author}{S.~Messoudi}, \bibinfo{author}{S.~Destercke}, \bibinfo{author}{S.~Rousseau},
\newblock \bibinfo{title}{Copula-based conformal prediction for multi-target regression},
\newblock \bibinfo{journal}{Pattern Recognition} \bibinfo{volume}{120} (\bibinfo{year}{2021}) \bibinfo{pages}{108101}.
\bibitem[{Diquigiovanni et~al.(2022)Diquigiovanni, Fontana, and Vantini}]{diquigiovanni2022conformal}
\bibinfo{author}{J.~Diquigiovanni}, \bibinfo{author}{M.~Fontana}, \bibinfo{author}{S.~Vantini},
\newblock \bibinfo{title}{Conformal prediction bands for multivariate functional data},
\newblock \bibinfo{journal}{Journal of Multivariate Analysis} \bibinfo{volume}{189} (\bibinfo{year}{2022}) \bibinfo{pages}{104879}.
\bibitem[{Xu et~al.(2024)Xu, Jiang, and Xie}]{xu2024conformal}
\bibinfo{author}{C.~Xu}, \bibinfo{author}{H.~Jiang}, \bibinfo{author}{Y.~Xie},
\newblock \bibinfo{title}{Conformal prediction for multi-dimensional time series by ellipsoidal sets},
\newblock in: \bibinfo{booktitle}{Forty-first International Conference on Machine Learning}, \bibinfo{year}{2024}.
\bibitem[{Johnstone and Cox(2021)}]{johnstone2021conformal}
\bibinfo{author}{C.~Johnstone}, \bibinfo{author}{B.~Cox},
\newblock \bibinfo{title}{Conformal uncertainty sets for robust optimization},
\newblock in: \bibinfo{booktitle}{Conformal and Probabilistic Prediction and Applications}, \bibinfo{organization}{PMLR}, \bibinfo{year}{2021}, pp. \bibinfo{pages}{72--90}.
\bibitem[{Colombo(2024)}]{colombo2024normalizing}
\bibinfo{author}{N.~Colombo},
\newblock \bibinfo{title}{Normalizing flows for conformal regression},
\newblock in: \bibinfo{booktitle}{The 40th Conference on Uncertainty in Artificial Intelligence}, \bibinfo{year}{2024}.
\bibitem[{Rigollet and Vert(2009)}]{rigollet2009optimal}
\bibinfo{author}{P.~Rigollet}, \bibinfo{author}{R.~Vert},
\newblock \bibinfo{title}{Optimal rates for plug-in estimators of density level sets},
\newblock \bibinfo{journal}{Bernoulli} \bibinfo{volume}{15} (\bibinfo{year}{2009}) \bibinfo{pages}{1154--1178}.
\bibitem[{Samworth and Wand(2010)}]{samworth2010asymptotics}
\bibinfo{author}{R.~Samworth}, \bibinfo{author}{M.~Wand},
\newblock \bibinfo{title}{Asymptotics and optimal bandwidth selection for highest density region estimation},
\newblock \bibinfo{journal}{The Annals of Statistics} \bibinfo{volume}{38} (\bibinfo{year}{2010}) \bibinfo{pages}{1767--1792}.
\bibitem[{Izbicki et~al.(2022)Izbicki, Shimizu, and Stern}]{izbicki2022cd}
\bibinfo{author}{R.~Izbicki}, \bibinfo{author}{G.~Shimizu}, \bibinfo{author}{R.~B. Stern},
\newblock \bibinfo{title}{Cd-split and hpd-split: Efficient conformal regions in high dimensions},
\newblock \bibinfo{journal}{Journal of Machine Learning Research} \bibinfo{volume}{23} (\bibinfo{year}{2022}) \bibinfo{pages}{1--32}.
\bibitem[{Sampson and Chan(2024)}]{sampson2024flexible}
\bibinfo{author}{M.~Sampson}, \bibinfo{author}{K.-S. Chan},
\newblock \bibinfo{title}{Flexible conformal highest predictive conditional density sets},
\newblock \bibinfo{journal}{arXiv preprint arXiv:2406.18052}  (\bibinfo{year}{2024}).
\bibitem[{Cai et~al.(2014)Cai, Low, and Ma}]{cai2014adaptive}
\bibinfo{author}{T.~T. Cai}, \bibinfo{author}{M.~Low}, \bibinfo{author}{Z.~Ma},
\newblock \bibinfo{title}{Adaptive confidence bands for nonparametric regression functions},
\newblock \bibinfo{journal}{Journal of the American Statistical Association} \bibinfo{volume}{109} (\bibinfo{year}{2014}) \bibinfo{pages}{1054--1070}.
\bibitem[{Wang et~al.(2023)Wang, Gao, Yin, Zhou, and Blei}]{wang2023probabilistic}
\bibinfo{author}{Z.~Wang}, \bibinfo{author}{R.~Gao}, \bibinfo{author}{M.~Yin}, \bibinfo{author}{M.~Zhou}, \bibinfo{author}{D.~Blei},
\newblock \bibinfo{title}{Probabilistic conformal prediction using conditional random samples},
\newblock in: \bibinfo{booktitle}{International Conference on Artificial Intelligence and Statistics}, \bibinfo{organization}{PMLR}, \bibinfo{year}{2023}, pp. \bibinfo{pages}{8814--8836}.
\bibitem[{Han et~al.(2022)Han, Tang, Ghosh, and Liu}]{han2022split}
\bibinfo{author}{X.~Han}, \bibinfo{author}{Z.~Tang}, \bibinfo{author}{J.~Ghosh}, \bibinfo{author}{Q.~Liu},
\newblock \bibinfo{title}{Split localized conformal prediction},
\newblock \bibinfo{journal}{arXiv preprint arXiv:2206.13092}  (\bibinfo{year}{2022}).
\bibitem[{Sesia and Romano(2021)}]{sesia2021conformal}
\bibinfo{author}{M.~Sesia}, \bibinfo{author}{Y.~Romano},
\newblock \bibinfo{title}{Conformal prediction using conditional histograms},
\newblock \bibinfo{journal}{Advances in Neural Information Processing Systems} \bibinfo{volume}{34} (\bibinfo{year}{2021}) \bibinfo{pages}{6304--6315}.
\bibitem[{Plassier et~al.(2025)Plassier, Fishkov, Guizani, Panov, and Moulines}]{plassier2025conditionally}
\bibinfo{author}{V.~Plassier}, \bibinfo{author}{A.~Fishkov}, \bibinfo{author}{M.~Guizani}, \bibinfo{author}{M.~Panov}, \bibinfo{author}{E.~Moulines},
\newblock \bibinfo{title}{Probabilistic conformal prediction with approximate conditional validity},
\newblock in: \bibinfo{booktitle}{The Thirteenth International Conference on Learning Representations}, \bibinfo{year}{2025}.
\bibitem[{Zheng and Zhu(2024)}]{zheng2024optimizing}
\bibinfo{author}{M.~Zheng}, \bibinfo{author}{S.~Zhu},
\newblock \bibinfo{title}{Optimizing probabilistic conformal prediction with vectorized non-conformity scores},
\newblock \bibinfo{journal}{arXiv preprint arXiv:2410.13735}  (\bibinfo{year}{2024}).
\bibitem[{Tumu et~al.(2024)Tumu, Cleaveland, Mangharam, Pappas, and Lindemann}]{tumu2024multi}
\bibinfo{author}{R.~Tumu}, \bibinfo{author}{M.~Cleaveland}, \bibinfo{author}{R.~Mangharam}, \bibinfo{author}{G.~Pappas}, \bibinfo{author}{L.~Lindemann},
\newblock \bibinfo{title}{Multi-modal conformal prediction regions by optimizing convex shape templates},
\newblock in: \bibinfo{booktitle}{6th Annual Learning for Dynamics \& Control Conference}, \bibinfo{organization}{PMLR}, \bibinfo{year}{2024}, pp. \bibinfo{pages}{1343--1356}.
\bibitem[{Thurin et~al.(2025)Thurin, Nadjahi, and Boyer}]{thurin2025optimal}
\bibinfo{author}{G.~Thurin}, \bibinfo{author}{K.~Nadjahi}, \bibinfo{author}{C.~Boyer},
\newblock \bibinfo{title}{Optimal transport-based conformal prediction},
\newblock \bibinfo{journal}{arXiv preprint arXiv:2501.18991}  (\bibinfo{year}{2025}).
\bibitem[{Yang and Kuchibhotla(2024)}]{yang2024selection}
\bibinfo{author}{Y.~Yang}, \bibinfo{author}{A.~K. Kuchibhotla},
\newblock \bibinfo{title}{Selection and aggregation of conformal prediction sets},
\newblock \bibinfo{journal}{Journal of the American Statistical Association}  (\bibinfo{year}{2024}) \bibinfo{pages}{1--13}.
\bibitem[{Luo and Zhou(2024)}]{luo2024weighted}
\bibinfo{author}{R.~Luo}, \bibinfo{author}{Z.~Zhou},
\newblock \bibinfo{title}{Weighted aggregation of conformity scores for classification},
\newblock \bibinfo{journal}{arXiv preprint arXiv:2407.10230}  (\bibinfo{year}{2024}).
\bibitem[{Fang et~al.(2025)Fang, Tan, and Huang}]{fang2025contra}
\bibinfo{author}{Z.~Fang}, \bibinfo{author}{A.~Tan}, \bibinfo{author}{J.~Huang},
\newblock \bibinfo{title}{{CONTRA}: Conformal prediction region via normalizing flow transformation},
\newblock in: \bibinfo{booktitle}{The Thirteenth International Conference on Learning Representations}, \bibinfo{year}{2025}.
\bibitem[{Papamakarios et~al.(2017)Papamakarios, Pavlakou, and Murray}]{papamakarios2017masked}
\bibinfo{author}{G.~Papamakarios}, \bibinfo{author}{T.~Pavlakou}, \bibinfo{author}{I.~Murray},
\newblock \bibinfo{title}{Masked autoregressive flow for density estimation},
\newblock \bibinfo{journal}{Advances in neural information processing systems} \bibinfo{volume}{30} (\bibinfo{year}{2017}).

\end{thebibliography}
\end{document}